\documentclass[letterpaper,10pt]{article}

\usepackage[preprint]{probml}

\ShortHeadings{Conditionally Identifiable Latent Representation for Time Series}

\usepackage[american]{babel}
\usepackage{mathtools}
\usepackage{amsfonts}
\usepackage{bm}
\usepackage{graphicx}
\graphicspath{{figures/}}
\usepackage{url}
\usepackage{booktabs}
\usepackage{multirow}
\usepackage{subcaption}

\IfFileExists{algorithm.sty}{%
  \usepackage{algorithm}%
  \usepackage{algpseudocode}%
}{%
}
\usepackage{float}
\usepackage{rotating}

\def\1{\bm{1}}

\DeclareMathAlphabet{\mathsfit}{\encodingdefault}{\sfdefault}{m}{sl}
\SetMathAlphabet{\mathsfit}{bold}{\encodingdefault}{\sfdefault}{bx}{n}

\begin{document}

\title{Conditionally Identifiable Latent Representation for Multivariate Time Series with Structural Dynamics}

\author[1]{Minkey Chang, Jae Young Kim}
\affil[1]{Seoul National University}

\maketitle

\makeatletter
\if@preprint
\begin{center}\small\itshape
Published in \emph{Proceedings of The 1st Symposium on Probabilistic Machine Learning} (ProbML 2026), PMLR 327:189--213.\\
Code and experiment scripts: \url{https://github.com/minkeymouse/iVDFM}
\end{center}
\fi
\makeatother

\begin{abstract}
We propose the Identifiable Variational Dynamic Factor Model (iVDFM) for multivariate time series. The model combines variational inference with \emph{partial} identifiability guarantees up to a known ambiguity class. Innovations follow a conditional exponential-family prior over observed auxiliary context and regime embeddings. Linear diagonal dynamics map innovations to factors and preserve that class, so factors are partially identifiable up to permutation and component-wise affine maps. We train by maximizing the ELBO and estimate uncertainty in both latent trajectories and forecasts. Under explicit assumptions, we prove partial identifiability up to the stated class and provide the full proof in the appendix. We evaluate factor recovery on synthetic DGPs, intervention behavior on synthetic SCMs, and forecasting on real benchmarks with CRPS and MSE.
\end{abstract}

\section{Introduction}
\label{sec:introduction}

Time-series modeling often requires assigning uncertainty and making downstream intervention queries in a latent space. When latent coordinates are not identifiable up to a controlled ambiguity class, the same observations admit multiple latent factorizations with different coordinate semantics, so uncertainty and intervention statements become ill-defined. Our goal is therefore to learn latent factors whose meaning is stable up to a known, small ambiguity class of transformations.

Classical dynamic factor models recover temporal structure but are identifiable only up to arbitrary invertible linear mixings of factors \citep{stock2002}. Identifiable VAEs (iVAEs) reduce ambiguity in static settings to permutation and component-wise affine maps by conditioning the latent prior on auxiliary context \citep{khemakhem2020iva}, but they do not directly target sequential models. We propose the \emph{Identifiable Variational Dynamic Factor Model (iVDFM)}, which shifts the identification target to the \emph{innovation process} that drives the dynamics: we impose a conditional non-Gaussian exponential-family prior on innovations and propagate them to factors through linear diagonal dynamics. Under explicit assumptions, this yields \emph{partial} identifiability of both innovations and factor trajectories up to permutation and component-wise affine maps.

iVDFM is the dynamic-factor counterpart of the iVAE construction of \citet{khemakhem2020iva}: innovations $\boldsymbol\eta_{1:T}$ take the role of their sources $\mathbf z$; the pair $(\mathbf u_t, \mathbf e_t)$, with $\mathbf e_t = \psi(\mathbf u_t)$ a deterministic regime embedding, takes the role of the auxiliary $\mathbf u$; and the temporal observation map $\boldsymbol y_t = g(\mathcal F_t(\boldsymbol\eta_{1:t})) + \boldsymbol\varepsilon_t$, composing the learned dynamics $\mathcal F_t$ with a decoder $g$, replaces the static mixing $\mathbf x = \mathbf f(\mathbf z)$. Innovation-level identifiability then follows from \citet[Theorem 1]{khemakhem2020iva} in this correspondence; the step we add is carrying the ambiguity class $\mathcal T$ through the dynamics to factor trajectories, and assembling a dynamic factor model around this core.

The technical content of that carry-through is a propagation lemma (Lemma~\ref{lem:propagation}): under diagonal or block-diagonal companion-form dynamics, $\mathcal{T}$ transfers from innovations to factor trajectories, and the same implication fails under dense dynamics or Gaussian innovations. On the inference side, the ELBO factorizes along the Markov structure and training proceeds on sliding windows, which keeps identification (a property of the generative prior) separate from the variational approximation (a property of the encoder). The regime embedding is a deterministic mixture of the auxiliary, which extends the conditional innovation family while preserving exogeneity. Factor recovery on synthetic dynamic DGPs, intervention analysis on synthetic SCMs, and probabilistic forecasting on real benchmarks then show that innovation-level partial identifiability is compatible with competitive predictive performance.
\section{Related Work}
\label{sec:related_work}

Latent-variable models are typically identifiable only up to an \emph{ambiguity class} of transformations that leave the observation law unchanged; the practical question is how small that class can be made. Classical methods such as PCA \citep{pearson1901pca} and ICA \citep{hyvarinen2000ica} recover structured directions under strong covariance or independence assumptions that do not extend cleanly to nonlinear or noisy settings. VAEs \citep{kingma2014vae} relax those assumptions but in exchange leave the latent coordinates broadly non-unique, and identifiable VAEs \citep{khemakhem2020iva} recover a small class in the static case by conditioning the latent prior on auxiliary variables.

For multivariate time series, classical dynamic factor models (DFMs) \citep{stock2002} summarize many observed variables with linear latent dynamics, but factors are identified only up to arbitrary invertible linear mixings (a rotational indeterminacy in the Gaussian case; see Appendix \ref{app:gaussian_degeneracy}). Structural VAR approaches \citep{sims1980,svar_external_instrument2020} identify interpretable shocks via restrictions or instruments, though they rely on linear assumptions.

Recent deep sequence models such as iTransformer \citep{liu2023itransformer}, TimeMixer \citep{wang2024timemixer}, S4 \citep{gu2022efficiently}, and Mamba \citep{gu2023mamba} emphasize predictive accuracy more than latent identifiability. Deep Dynamic Factor Models \citep{andreini2020deep} and Deep Kalman Filters \citep{krishnan2015dkf} combine probabilistic temporal structure with nonlinear mappings, but they typically do not provide a narrow latent ambiguity class. Work on auxiliary-conditioned sequential identifiability \citep{hyvarinen2019auxiliary,balsellsrodas2024switching} shows how temporal structure and auxiliary variation can shrink ambiguity under specific assumptions.

A useful lens on this landscape is the size of the residual ambiguity group. Gaussian DFMs and linear state-space models are identified only up to the full general linear group of invertible $r \times r$ real matrices, $\mathrm{GL}(r)$, acting on factors; Gaussian linear dynamical systems (LDSs) stay inside this group because the Gaussian family is closed under invertible linear maps (Appendix \ref{app:gaussian_degeneracy}). Static iVAEs and their non-Gaussian extensions \citep{khemakhem2020iva,hyvarinen2019auxiliary} shrink the group to the much smaller class $\mathcal{T}$ of permutations and component-wise affine maps, contingent on assumptions about the auxiliary variable. iVDFM inherits $\mathcal{T}$ at the innovation level and, via Lemma \ref{lem:propagation}, preserves $\mathcal{T}$ at the factor level through diagonal dynamics, trading the Gaussian closure and dense dynamics of classical DFMs for non-Gaussian innovations and component-wise (or block-diagonal companion) propagation. Switching-state identifiable models \citep{balsellsrodas2024switching} also reach a sub-$\mathrm{GL}(r)$ ambiguity, but through discrete regime switches; iVDFM instead uses a deterministic regime mixture that preserves exogeneity of the auxiliary variable without committing to hard transitions.

From a causal perspective, this matters because intervention semantics depend on stable latent coordinates. If latent factors are defined only up to a broad class, causal conclusions can vary across equivalent parameterizations \citep{scholkopf2021causal,spirtes2000causation}. Identifiability up to a known and controlled class is therefore a prerequisite for well-defined intervention queries \citep{pearl2009causality,peters2017elements}.
\section{Identifiable Variational Dynamic Factor Model}
\label{sec:method}

The \emph{Identifiable Variational Dynamic Factor Model (iVDFM)} concentrates all latent stochasticity in innovations $\boldsymbol{\eta}_t$ drawn from a conditional prior and propagates them to factors through diagonal time-varying linear dynamics, with parameters estimated by variational inference (Figure \ref{fig:model_schematic}). Diagonality makes factor trajectories deterministic functions of the innovation history, so we can analyze identifiability at the innovation level and carry it over to factors.

\begin{figure*}[t]
\centering
\includegraphics[width=0.95\linewidth, trim=0 45 0 0, clip]{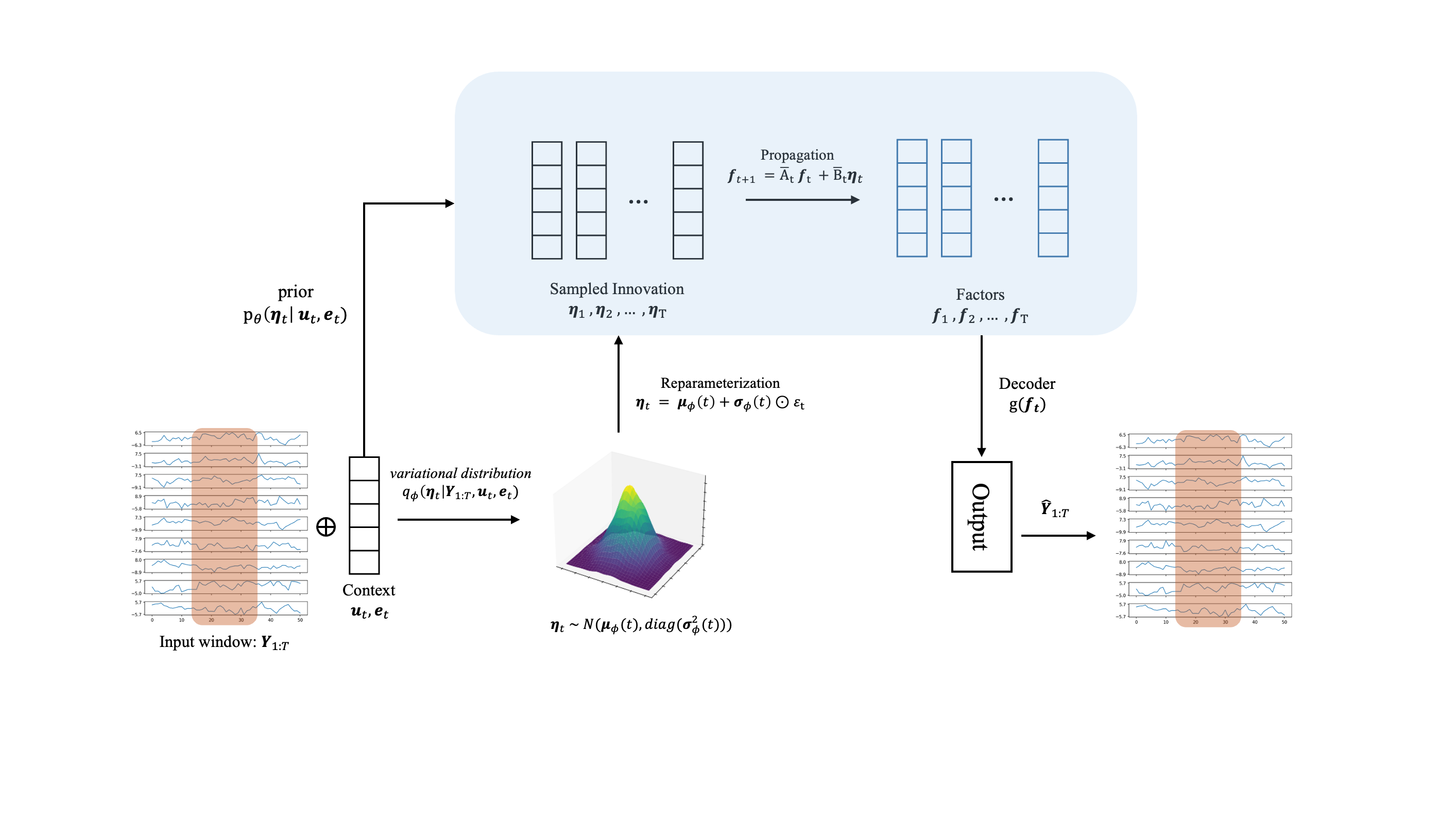}
\caption{iVDFM in three steps. (a) Extraction. Observations $\boldsymbol{y}_{1:T}$ and context $(\mathbf{u}_t, \mathbf{e}_t)$ are mapped by the encoder to a variational distribution $q_\phi(\boldsymbol{\eta}_t \mid \cdot)$ over innovations, while the prior $p_\theta(\boldsymbol{\eta}_t \mid \mathbf{u}_t, \mathbf{e}_t)$ conditions on the same context. (b) Propagation. Innovations $\boldsymbol{\eta}_t$ drive factors $\boldsymbol{f}_t$ through diagonal dynamics $\bar{A}_t$, $\bar{B}_t$. (c) Forecasting and generation. Factors are decoded by $g$ to reconstructed or forecast observations $\widehat{\boldsymbol{y}}_t$.}
\label{fig:model_schematic}
\end{figure*}

The generative model then reduces to three equations,
\begin{align}
\boldsymbol{\eta}_t &\sim p_\theta(\cdot \mid \mathbf{u}_t, \mathbf{e}_t) && \text{(conditional exponential-family innovation prior)},\nonumber\\
\boldsymbol{f}_{t+1} &= \bar{A}_t \boldsymbol{f}_t + \bar{B}_t \boldsymbol{\eta}_t && \text{(diagonal or block-diagonal companion-form dynamics)},\nonumber\\
\boldsymbol{y}_t &= g(\boldsymbol{f}_t) + \boldsymbol{\varepsilon}_t && \text{(injective MLP decoder with fixed-scale noise)},\nonumber
\end{align}
whose components we describe before stating the identifiability result (Theorem \ref{thm:temporal_identifiability}).

\subsection{Innovation}

The innovation $\boldsymbol{\eta}_t \in \mathbb{R}^r$ carries a conditional non-Gaussian exponential-family prior that factorizes across components,
\begin{equation}
\label{eq:innovation_prior}
p(\boldsymbol{\eta}_t \mid \mathbf{u}_t, \mathbf{e}_t)
= \prod_{i=1}^r h_i(\eta_{i,t}) \exp\!\Bigl(\mathbf{T}_i(\eta_{i,t})^\top \boldsymbol{\lambda}_i(\mathbf{u}_t, \mathbf{e}_t) - \Psi_i\!\left(\boldsymbol{\lambda}_i(\mathbf{u}_t, \mathbf{e}_t)\right)\Bigr),
\end{equation}
with sufficient statistics $\mathbf{T}_i \in \mathbb{R}^{k}$, natural parameters $\boldsymbol{\lambda}_i(\mathbf{u}_t, \mathbf{e}_t) \in \mathbb{R}^{k}$, base measure $h_i$, and log-partition $\Psi_i$. Attaching identifiability to this prior, rather than to factors directly, keeps a clean separation between what is identified (innovation coordinates) and how it is used (propagation and decoding): sufficient variation of the natural-parameter map $\boldsymbol{\lambda}(\mathbf{u}_t, \mathbf{e}_t)$ across the auxiliary $(\mathbf{u}_t, \mathbf{e}_t)$ identifies innovations up to permutation and component-wise affine maps \citep{hyvarinen2000ica,khemakhem2020iva}.

The regime context $\mathbf{e}_t$ is a deterministic soft mixture,
\[
\mathbf{e}_t = \sum_{j=1}^K \pi_{t,j} \mathbf{e}_j,\qquad
\boldsymbol{\pi}_t = \mathrm{softmax}(\mathrm{RegimeNet}(\mathbf{u}_t)),
\]
over learnable embeddings $\mathbf{e}_j \in \mathbb{R}^d$ produced by a small MLP; the pair $(\mathbf{u}_t, \mathbf{e}_t)$ plays the role of the iVAE auxiliary, and taking the expected embedding rather than sampling a discrete regime keeps $\mathbf{e}_t$ a deterministic function of $\mathbf{u}_t$. The induced prior $\sum_j \pi_{t,j}\, p_j(\boldsymbol{\eta} \mid \mathbf{u}_t)$ with regime-specific natural parameters $\boldsymbol{\lambda}_j(\mathbf{u}_t)=\boldsymbol{\lambda}(\mathbf{u}_t,\mathbf{e}_j)$ enlarges the family of conditional innovation distributions reachable by a single $\boldsymbol{\lambda}(\mathbf{u}_t)$. Only the weighted combination of the regime parameters $(\mathbf{e}_j, A_j, B_j)$ enters the observation law, so individual regime parameters are not separately resolved; the identification guarantee applies to innovations and factors, not to the components of the mixture. The determinism requirement extends to any auxiliary one might plug in: calendar features, known event indicators, and similar domain metadata fold into $\mathbf{u}_t$ without issue, while stochastic context or variables computed from $\boldsymbol{y}_t$ forfeit exogeneity.

Identifiability at the innovation level rests on how much $\boldsymbol{\lambda}(\mathbf{u}, \mathbf{e})$ varies with the auxiliary. The natural-parameter map is required to have full column rank at $rk+1$ distinct auxiliary values \citep{khemakhem2020iva}; for scalar sufficient statistics ($k{=}1$, as in the Laplace innovations we use) this reduces to $r+1$. Richer auxiliary variation sharpens the conclusion, while weaker variation leaves room for distinct parameterizations to induce the same observation law.

\subsection{Dynamics}

Factors propagate from innovations through diagonal dynamics,
\begin{equation}
\label{eq:dynamics_baseline}
\boldsymbol{f}_{t+1} = \bar{A}_t \boldsymbol{f}_t + \bar{B}_t \boldsymbol{\eta}_t,
\end{equation}
with diagonal $\bar{A}_t, \bar{B}_t \in \mathbb{R}^{r \times r}$ formed by the same regime weights, $\bar{A}_t = \sum_j \pi_{t,j} A_j$ and $\bar{B}_t = \sum_j \pi_{t,j} B_j$, each $A_j, B_j$ diagonal. Each factor coordinate then depends only on its own history and its own innovation, so the innovation-level ambiguity class $\mathcal{T}$ transfers unchanged to factor trajectories. AR($p$) dependence fits into the same picture via the standard companion-form stacking \citep{lutkepohl2005,hamilton1994}, which preserves block-diagonality across factor coordinates and hence the same component-wise propagation (Appendix \ref{app:A:companion}).

\begin{definition}[Identifiability up to a class]
\label{def:identifiability}
Innovations (or factors) are \emph{identifiable up to a class of transformations} $\mathcal{T}$ if, for any two parameterizations that induce the same $p(\boldsymbol{y}_{1:T}, \boldsymbol{u}_{1:T})$, the corresponding latent variables are related by a transformation in $\mathcal{T}$ almost surely.
\end{definition}

Concretely, any $\tau \in \mathcal{T}$ acts as $\tilde\eta_{i,t} = a_{\pi(i)}\, \eta_{\pi(i),t} + b_{\pi(i)}$ for a permutation $\pi$ and component-wise scalars $a_i \neq 0$, $b_i$; there is no cross-coordinate mixing. So the recovered coordinates name the same directions up to relabeling, sign, shift, and per-coordinate scale, which is what lets intervention and uncertainty statements remain stable under equivalent parameterizations. This is the qualitative gap from rotationally ambiguous Gaussian factor models, where actions and explanations sit in a coordinate system with a much larger ambiguity group.

\begin{lemma}[Propagation of $\mathcal{T}$ through diagonal dynamics]
\label{lem:propagation}
Let $\boldsymbol{f}_t = \mathcal{F}_t(\boldsymbol{\eta}_{1:t})$ denote the deterministic innovation-to-factor map induced by the diagonal dynamics in Eq.\ \ref{eq:dynamics_baseline} or the block-diagonal companion form. If innovations are identified up to $\mathcal{T}$, then so are factor trajectories: any reparameterization $\tilde{\boldsymbol{\eta}}_t = \tau(\boldsymbol{\eta}_t)$ with $\tau \in \mathcal{T}$ induces a factor reparameterization $\tilde{\boldsymbol{f}}_t = \bar{\tau}_t(\boldsymbol{f}_t)$ with $\bar{\tau}_t \in \mathcal{T}$ for all $t$.
\end{lemma}
The proof is in Appendix \ref{app:A:proof} (Step 2); the key property is that $\mathcal{T}$ is preserved, not enlarged, when innovations pass through the dynamics. The same argument breaks down under dense (non-diagonal) dynamics or Gaussian innovations, where additional symmetries in the induced observation law enlarge the ambiguity group (Appendix \ref{app:gaussian_degeneracy}).

\begin{theorem}[Identifiability of dynamic factors]
\label{thm:temporal_identifiability}
Consider the generative model of this section (innovation prior Eq.\ \ref{eq:innovation_prior}, diagonal dynamics Eq.\ \ref{eq:dynamics_baseline} or block-diagonal companion form, observation model with injective $g$ and fixed noise). Assume:
\begin{enumerate}
\item The innovation prior is conditional exponential-family with linearly independent sufficient statistics $\mathbf{T}_i \in \mathbb{R}^{k}$, and the natural-parameter map $\boldsymbol{\lambda}(\mathbf{u}, \mathbf{e})$ has full column rank for at least $rk+1$ distinct values of $(\mathbf{u}, \mathbf{e})$; with scalar sufficient statistics ($k=1$) this reduces to $r+1$.
\item $\boldsymbol{\pi}_t$ depends only on $\mathbf{u}_t$, with diagonal $A_j$, $B_j$ per regime.
\item The decoder $g$ is injective on the support of $\boldsymbol{f}_t$.
\item Observation noise has full support and $\sigma^2$ is fixed.
\item Base measures have full support, $r \le N$, and all mappings are measurable and integrable.
\end{enumerate}
Then the innovations $\{\boldsymbol{\eta}_t\}_{t=1}^T$ are \emph{partially} identifiable from $(\boldsymbol{y}_{1:T}, \boldsymbol{u}_{1:T})$ up to $\mathcal{T}$ (Definition \ref{def:identifiability}, with $\mathcal{T}$ the class of permutation and component-wise affine maps). Combined with Lemma \ref{lem:propagation}, the latent factor trajectories $\{\boldsymbol{f}_t\}_{t=1}^T$ and the associated dynamic parameters are partially identifiable up to the same class.
\end{theorem}

The innovation-level part follows from \citet{khemakhem2020iva} applied with auxiliary $(\mathbf{u}_t, \mathbf{e}_t)$; Lemma \ref{lem:propagation} carries the conclusion to factor trajectories. The full proof is in Appendix \ref{app:identifiability}.

\subsection{Estimation}

Observations $\boldsymbol{y}_t \in \mathbb{R}^N$ arise from factors through an injective MLP decoder $g : \mathbb{R}^r \to \mathbb{R}^N$ with fixed-scale additive noise, $\boldsymbol{y}_t = g(\boldsymbol{f}_t) + \boldsymbol{\varepsilon}_t$. The posterior over innovations is approximated by a component-factorized Gaussian encoder,
\[
q_\phi(\boldsymbol{\eta}_t \mid \boldsymbol{y}_t, \mathbf{u}_t) = \mathcal{N}\!\left(\boldsymbol{\mu}_\phi(t), \operatorname{diag}(\boldsymbol{\sigma}_\phi^2(t))\right),
\]
which does not need $\mathbf{e}_t$ as an explicit input because it is already a deterministic function of $\mathbf{u}_t$. The encoder sees a single time step: identifiability is a property of the generative prior, so a longer-context encoder would not shrink $\mathcal{T}$, and restricting the context reduces the chance that the encoder routes information through $\mathbf{u}_t$ at the expense of $\boldsymbol{y}_t$. Sampling uses the standard reparameterization \citep{kingma2014vae}, $\boldsymbol{\eta}_t = \boldsymbol{\mu}_\phi(t) + \boldsymbol{\sigma}_\phi(t) \odot \boldsymbol{\epsilon}_t$ with $\boldsymbol{\epsilon}_t \sim \mathcal{N}(\boldsymbol{0}, I)$, and sampled innovations are propagated through the dynamics to form factors.

We train by maximizing the evidence lower bound (ELBO) jointly over generative parameters $\theta$ (decoder, prior including $\{\mathbf{e}_j\}$ and RegimeNet, dynamics) and encoder parameters $\phi$. Under the mean-field variational distribution $q_\phi(\boldsymbol{\eta}_{1:T}\mid\boldsymbol{y}_{1:T},\mathbf{u}_{1:T}) = \prod_{t=1}^T q_\phi(\boldsymbol{\eta}_t \mid \boldsymbol{y}_t, \mathbf{u}_t)$ and with factors deterministic given innovations, the bound factorizes across time:
\begin{align}
\mathcal{L}
&=
\mathbb{E}_{q_\phi(\boldsymbol{\eta}_{1:T}\mid\cdot)}\!\left[
\sum_{t=1}^T
\log p_\theta\!\left(\boldsymbol{y}_t \mid \mathcal{F}_t(\boldsymbol{\eta}_{1:t})\right)
\right]
\nonumber \\
&\quad
-
\sum_{t=1}^T
\mathrm{KL}\!\left(
q_\phi(\boldsymbol{\eta}_t \mid \boldsymbol{y}_t, \mathbf{u}_t)
\,\big\|\,
p_\theta(\boldsymbol{\eta}_t \mid \mathbf{u}_t, \mathbf{e}_t)
\right),
\label{eq:elbo}
\end{align}
with $\boldsymbol{f}_t = \mathcal{F}_t(\boldsymbol{\eta}_{1:t})$ unrolled from $\boldsymbol{f}_{t+1} = \bar{A}_t \boldsymbol{f}_t + \bar{B}_t \boldsymbol{\eta}_t$ and a learned initial state $\boldsymbol{f}_0$. We optimize $\mathcal{L}$ on sliding windows of length $T$ with stride $T/2$ (Appendix \ref{app:E:hyperparameters}); partial windows at the sequence boundary are dropped rather than zero-padded, so the dynamics are never unrolled on artificial context. Windowed estimation is a finite-data surrogate; Theorem \ref{thm:temporal_identifiability} characterizes the invariance class of the underlying generative family, so neither the windowing scheme nor local adaptation through $(\mathbf{u}_t, \mathbf{e}_t)$ enters the identifiability argument.
\section{Experiments}
\label{sec:experiments}

Our experiments probe three questions raised by the identifiability claim: whether iVDFM actually recovers latent factors on synthetic data with known ground truth, whether the learned innovation coordinates support $do$-interventions on a known structural causal model, and whether enforcing partial identifiability compromises forecasting performance on real benchmarks. Factor recovery compares against DFM, DDFM, VAE, and iVAE; forecasting adds TimeMixer and related deep baselines. Code, configurations, and protocols are in Appendix \ref{app:experiment_details}.

\subsection{Factor Recovery}

We split the evaluation into \emph{dynamic} data-generating processes (DGPs; AR dynamics driven by innovations) and \emph{static} DGPs, aligning recovered factors to ground truth via permutation-invariant maximum-correlation matching and reporting the mean correlation coefficient (MCC), subspace distance, smoothness, and Trace $R^2$ (Tables \ref{tab:synthetic_factor_recovery_dynamic}--\ref{tab:synthetic_factor_recovery_static}). All methods see the same settings ($T{=}200$, $N{=}20$, $r{=}5$) with matched simulation budgets, and we report mean $\pm$ standard deviation across 10 seeds so the static-versus-dynamic contrast is not a single-seed artefact.

\begin{table}[H]
\centering
\caption{Factor recovery on the dynamic DGP ($T{=}200$, $N{=}20$, $r{=}5$). Metrics are MCC, subspace distance, smoothness, and Trace $R^2$ after MCC-based matching. Higher values are better for MCC and Trace $R^2$, while lower values are better for subspace distance and smoothness. Reported values are mean $\pm$ std over simulations.}
\label{tab:synthetic_factor_recovery_dynamic}
\begin{tabular}{lcccc}
\toprule
Model & MCC $\uparrow$ & Subspace $\downarrow$ & Smoothness $\downarrow$ & Trace $R^2$ $\uparrow$ \\
\midrule
iVDFM & \textbf{0.6548 $\pm$ 0.0549} & \textbf{0.4345 $\pm$ 0.1022} & 1.7392 $\pm$ 0.1118 & \textbf{0.8202 $\pm$ 0.0603} \\
DDFM & 0.5202 $\pm$ 0.0360 & 0.8441 $\pm$ 0.0564 & 2.0060 $\pm$ 0.1068 & 0.4998 $\pm$ 0.0435 \\
iVAE & \underline{0.6146 $\pm$ 0.0560} & \underline{0.5149 $\pm$ 0.1264} & 1.5864 $\pm$ 0.0710 & \underline{0.7664 $\pm$ 0.0871} \\
VAE & 0.6065 $\pm$ 0.0587 & 0.5418 $\pm$ 0.1286 & \underline{1.5844 $\pm$ 0.0741} & 0.7526 $\pm$ 0.0838 \\
DFM & 0.4681 $\pm$ 0.0310 & 0.9565 $\pm$ 0.0414 & \textbf{0.3149 $\pm$ 0.0566} & 0.3783 $\pm$ 0.0324 \\
\bottomrule
\end{tabular}
\end{table}

\begin{table}[H]
\centering
\caption{Factor recovery, static DGP ($T{=}200$, $N{=}20$, $r{=}5$). Metrics are computed after MCC-based matching; we report MCC, subspace distance, smoothness, and Trace $R^2$ (mean $\pm$ std over simulations).}
\label{tab:synthetic_factor_recovery_static}
\begin{tabular}{lcccc}
\toprule
Model & MCC $\uparrow$ & Subspace $\downarrow$ & Smoothness $\downarrow$ & Trace $R^2$ $\uparrow$ \\
\midrule
iVDFM & 0.5683 $\pm$ 0.0661 & 0.6949 $\pm$ 0.0544 & \underline{1.7930 $\pm$ 0.1281} & 0.6178 $\pm$ 0.0460 \\
DDFM & \underline{0.5989 $\pm$ 0.0673} & 0.6147 $\pm$ 0.0540 & 2.3296 $\pm$ 0.0911 & 0.6745 $\pm$ 0.0502 \\
iVAE & 0.5678 $\pm$ 0.0751 & \underline{0.5055 $\pm$ 0.0953} & 2.2504 $\pm$ 0.1186 & \underline{0.7766 $\pm$ 0.0589} \\
VAE & \textbf{0.6405 $\pm$ 0.0945} & \textbf{0.4520 $\pm$ 0.0675} & 2.3255 $\pm$ 0.1438 & \textbf{0.8078 $\pm$ 0.0472} \\
DFM & 0.3374 $\pm$ 0.0397 & 1.1192 $\pm$ 0.0448 & \textbf{0.0994 $\pm$ 0.0173} & 0.2476 $\pm$ 0.0446 \\
\bottomrule
\end{tabular}
\end{table}

On the dynamic DGP, iVDFM attains the highest MCC and Trace $R^2$ (MCC $0.6548 \pm 0.0549$; Trace $R^2$ $0.8202 \pm 0.0603$), with the gain visible at both coordinate alignment (MCC) and trajectory-level explanatory structure (Trace $R^2$). On the static DGP, VAE achieves the best MCC and Trace $R^2$ ($0.6405 \pm 0.0945$; $0.8078 \pm 0.0472$); iVDFM matches iVAE on MCC ($0.5683$ vs.\ $0.5678$) and retains stronger smoothness than the other nonlinear baselines. The gap between the two settings is where the identification signal lives: temporal innovation structure is what iVDFM exploits, so the dynamic DGP is where the mechanism is active, while on the static DGP there is nothing for the dynamic side of the model to add.

\subsection{Causal Intervention}

\begin{table}[H]
    \centering
    \small
    \caption{Causal intervention on synthetic SCMs (iVDFM). IRF-MSE, IRF-MAE: mean squared / absolute error between ground-truth and model-implied impulse response; Sign Acc: fraction of (time step, variable) pairs with correct sign; IRF Corr: Pearson correlation with ground truth. Mean $\pm$ std over five simulations.}
    \label{tab:causal_intervention}
    \begin{tabular}{lcccc}
    \toprule
    SCM & IRF-MSE $\downarrow$ & IRF-MAE $\downarrow$ & Sign Acc $\uparrow$ & IRF Corr $\uparrow$ \\
    \midrule
    Base (linear) & $0.89 \pm 1.01$ & $0.45 \pm 0.23$ & $0.58 \pm 0.25$ & $0.21 \pm 0.75$ \\
    Regime & $1.28 \pm 0.92$ & $0.77 \pm 0.27$ & $0.60 \pm 0.27$ & $0.13 \pm 0.62$ \\
    Chain & $2.18 \pm 2.40$ & $0.97 \pm 0.40$ & $0.58 \pm 0.33$ & $0.20 \pm 0.62$ \\
    \bottomrule
    \end{tabular}
\end{table}

To validate that the learned representation supports $do$-interventions at the innovation level, we generate synthetic series from a structural causal model (SCM), fit iVDFM on observational samples, and then apply $do$ by fixing one innovation component at a target time and propagating its effect through the learned dynamics and decoder. We compare the model-implied impulse response functions (IRFs) with ground-truth IRFs using IRF-MSE, sign accuracy, and IRF correlation (Table \ref{tab:causal_intervention}).

The chain setting is the hardest case because its dynamics violate diagonality: iVDFM keeps IRF-MSE bounded but loses sign and correlation fidelity. The degradation tracks the argument behind Lemma \ref{lem:propagation} — propagation relies on cross-factor independence — and points to richer transition structures as the natural relaxation for intervention-faithful dynamics.

\subsection{Forecasting}
\label{sec:probabilistic_forecasting}

We evaluate distributional forecasts on ETTh1/2, ETTm1/2, and Weather at horizons 96--720 using the continuous ranked probability score \citep[CRPS;][]{matheson1976scoring,gneiting2007strictly} and standardized MSE with rolling origins (Table \ref{tab:forecasting}).

\begin{table}[t]
\centering
\caption{CRPS and MSE (standardized), averaged over horizons. Rows: dataset; columns: per-model CRPS and MSE (std.).}
\label{tab:forecasting}
\tiny
\setlength{\tabcolsep}{2.5pt}
\renewcommand{\arraystretch}{0.95}
\resizebox{0.88\textwidth}{!}{%
\begin{tabular}{lcccccccccc}
\toprule
Dataset & \multicolumn{2}{c}{iVDFM} & \multicolumn{2}{c}{iTrans.} & \multicolumn{2}{c}{TimeMix} & \multicolumn{2}{c}{TimeXer} & \multicolumn{2}{c}{DDFM} \\
\cmidrule(lr){2-3} \cmidrule(lr){4-5} \cmidrule(lr){6-7} \cmidrule(lr){8-9} \cmidrule(lr){10-11}
 & CRPS & MSE & CRPS & MSE & CRPS & MSE & CRPS & MSE & CRPS & MSE \\
\midrule
ETTh1 & 0.282 & 1.258 & 0.693 & 1.296 & 0.276 & 0.801 & 0.542 & 1.189 & 0.334 & 1.342 \\
ETTh2 & 0.268 & 1.016 & 0.501 & 1.018 & 0.230 & 0.401 & 0.452 & 0.990 & 0.278 & 1.013 \\
ETTm1 & 0.315 & 1.221 & 0.706 & 1.288 & 0.229 & 0.547 & 0.598 & 1.250 & 0.356 & 1.277 \\
ETTm2 & 0.340 & 1.098 & 0.644 & 1.330 & 0.232 & 0.434 & 0.531 & 1.395 & 0.460 & 2.308 \\
Weather & 0.271 & 1.184 & 0.316 & 0.703 & 0.163 & 0.341 & 0.281 & 0.707 & 0.164 & 0.439 \\
\bottomrule
\end{tabular}%
}
\end{table}

Factor order and regime count are tuned by validation forecast loss. Across datasets iVDFM's CRPS and standardized MSE sit alongside several strong baselines; the reading we draw from this is that the innovation-level constraints used to obtain partial identifiability are compatible with competitive distributional forecasting, while the latent space continues to carry the controlled ambiguity $\mathcal{T}$ used in the recovery and intervention analyses above.

\section{Application: Exchange Rates and Epidemiology}
\label{sec:example_fx}

As qualitative case studies we fit iVDFM to two real panels of multivariate time series. In both cases observations $\boldsymbol{y}_t$ are the stacked series (currencies or surveillance measurements), the auxiliary $\mathbf{u}_t$ carries calendar encodings and standardized trend features, and we interpret the learned loadings on observed variables. Coordinates are only identified up to $\mathcal{T}$, so the useful output is the \emph{structure} of the loadings rather than their raw sign or scale.

\subsection{Exchange Rates}

On daily exchange-rate series for AUD, JPY, CAD, CNY, and NZD against USD, a two-factor representation ($r{=}2$) is enough to separate broad co-movement from the part of the variation that is more series-specific. The second factor loads strongly and with the same sign on AUD, JPY, CAD, and NZD (roughly $0.85$ to $0.94$) while CNY remains weakly connected, which is consistent with a market-wide USD component concentrated on the freely moving rates in the panel. The first factor carries moderate negative correlations on the same currencies ($-0.48$ to $-0.60$) and is again weak on CNY, behaving as a secondary re-pricing direction rather than a second copy of the same common mode. Because the coordinates are tied to an innovation process with controlled ambiguity rather than to an unconstrained embedding, shifts along these trajectories can be discussed as candidate shocks up to $\mathcal{T}$. In practice, $\mathcal{T}$ preserves the separation between the market-wide direction and the secondary re-pricing direction across retraining runs, whereas a rotationally ambiguous embedding would allow those two directions to mix freely and leave shock-level interpretation ill-posed.

\subsection{Epidemiology}

For weekly influenza-like illness surveillance we again use two factors ($r{=}2$), but the structure of the loadings tells a different story. One latent trajectory moves closely with standardized outpatient visits and aligns strongly with both ILI rates (about $-0.69$, $-0.71$, and $-0.64$; Figure \ref{fig:illness_factors_timeseries}): when this coordinate shifts, all three surveillance series move together, as expected of a common epidemic-burden axis. The second coordinate is weaker on the two ILI percentages ($\sim\!-0.30$) yet non-negligible on outpatient visits ($\sim\!-0.40$), consistent with residual utilization-sensitive variation that the main burden axis does not fully absorb. For monitoring purposes the useful output is precisely this two-factor decomposition: one coordinate summarizes shared epidemic intensity while the other captures utilization-linked structure that a one-dimensional summary would hide. Partial identifiability is what makes this decomposition portable across retraining: the burden axis remains attachable to the ILI series up to $\mathcal{T}$, so downstream dashboards and alert rules do not need to re-learn which coordinate corresponds to which quantity after each refit.

\begin{figure}[t]
\centering
\includegraphics[width=\linewidth]{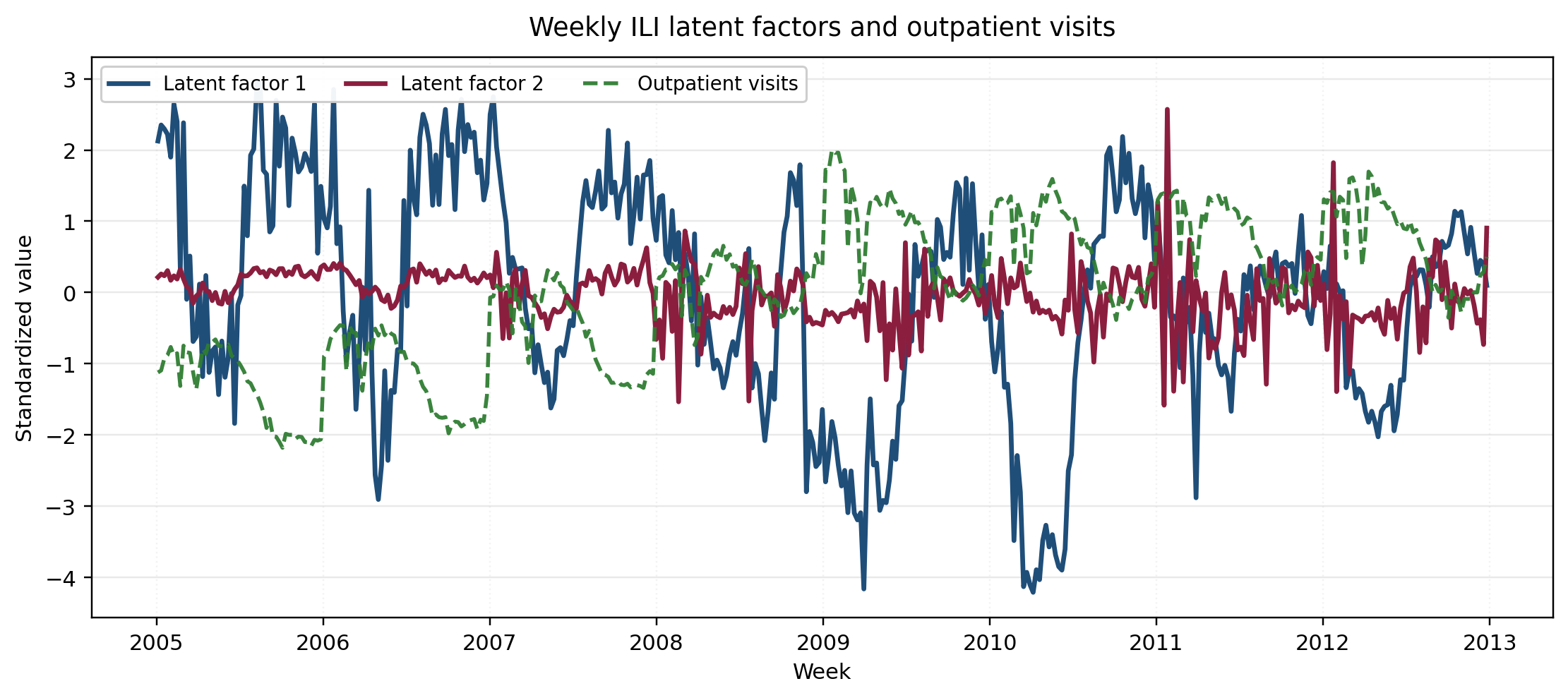}
\caption{Recovered iVDFM factors for the weekly illness application ($r{=}2$), with standardized outpatient visits overlay.}
\label{fig:illness_factors_timeseries}
\end{figure}

\section{Conclusion}
\label{sec:conclusion}

iVDFM extends variational latent modeling to a dynamic factor setting with conditional partial identifiability: under a conditional exponential-family innovation prior and diagonal dynamics, innovations and therefore factors are identified up to a controlled ambiguity class of permutations and component-wise affine maps, with inference carried out via the ELBO and efficient AR($p$) rollouts. Empirically, the model recovers factors on synthetic dynamic DGPs, matches intervention behavior on synthetic SCMs, and remains competitive on probabilistic forecasting benchmarks. Tying latent trajectories to innovation-level dynamics with an explicit ambiguity class makes downstream analysis more stable across runs and better aligned with intervention semantics (Section \ref{sec:example_fx} illustrates this on exchange-rate and epidemiology series).

A complementary design lever is the auxiliary $(\mathbf{u}_t, \mathbf{e}_t)$ itself. The theory only requires it to be a deterministic function of observed variables, which leaves room for substantially richer context than the calendar and trend features used here -- for instance language-model embeddings of event text or domain reports, which are deterministic encodings of observed text and can inject much more variation into $\boldsymbol{\lambda}(\mathbf{u}_t, \mathbf{e}_t)$ than low-dimensional covariates. The caveat is that if the underlying text is itself downstream of $\boldsymbol{y}_t$, treating it as exogenous silently breaks the auxiliary assumption; in that regime promoting the variable into $\boldsymbol{y}_t$ is the cleaner route. Viewed this way, Theorem \ref{thm:temporal_identifiability} gives a concrete criterion for deciding when a candidate context source should enter as auxiliary versus as observation.

The guarantees rest on fixed observation noise and diagonal dynamics, which are what keep the ambiguity group small. Relaxing these assumptions -- structured cross-factor interactions in the dynamics, or heteroskedastic or heavy-tailed observation noise -- is the natural next direction, at the cost of re-characterizing the residual symmetry group and sharpening the identifiability conditions. Within the current setting, the construction is already enough to support the innovation-level intervention queries of Section \ref{sec:experiments} and the panel-level interpretation of Section \ref{sec:example_fx}; the extensions above are what would be required when cross-factor feedback or heteroskedasticity is itself the object of study rather than a nuisance to bound.


\acks{We thank the reviewers for their careful reading and constructive feedback.}

\bibliography{references}

@article{pearson1901pca,
  title = {On lines and planes of closest fit to systems of points in space},
  author = {Pearson, Karl},
  journal = {Philosophical Magazine},
  volume = {2},
  number = {11},
  pages = {559--572},
  year = {1901}
}

@article{hyvarinen2000ica,
  title = {Independent component analysis: Algorithms and applications},
  author = {Hyvarinen, Aapo and Oja, Erkki},
  journal = {Neural Networks},
  volume = {13},
  number = {4--5},
  pages = {411--430},
  year = {2000}
}

@inproceedings{kingma2014vae,
  title = {Auto-Encoding Variational Bayes},
  author = {Kingma, Diederik P. and Welling, Max},
  booktitle = {International Conference on Learning Representations},
  year = {2014}
}

@inproceedings{khemakhem2020iva,
  author = {Khemakhem, Ilyes and Kingma, Diederik and Monti, Riccardo and Hyvarinen, Aapo},
  title = {Variational Autoencoders and Nonlinear ICA: A Unifying Framework},
  booktitle = {Proceedings of the 23rd International Conference on Artificial Intelligence and Statistics},
  year = {2020},
  volume = {108},
  pages = {2207--2217},
  url = {https://proceedings.mlr.press/v108/khemakhem20a.html}
}

@inproceedings{hyvarinen2019auxiliary,
  title = {Nonlinear {ICA} Using Auxiliary Variables and Generalized Contrastive Learning},
  author = {Hyvarinen, Aapo and Sasaki, Hiroaki and Turner, Richard},
  booktitle = {Proceedings of the Twenty-Second International Conference on Artificial Intelligence and Statistics},
  pages = {859--868},
  year = {2019},
  editor = {Chaudhuri, Kamalika and Sugiyama, Masashi},
  volume = {89},
  series = {Proceedings of Machine Learning Research},
  publisher = {PMLR},
  url = {https://proceedings.mlr.press/v89/hyvarinen19a.html}
}

@inproceedings{balsellsrodas2024switching,
  title = {On the Identifiability of Switching Dynamical Systems},
  author = {Balsells-Rodas, Carles and Wang, Yixin and Li, Yingzhen},
  booktitle = {Proceedings of the 41st International Conference on Machine Learning},
  pages = {2639--2672},
  year = {2024},
  editor = {Salakhutdinov, Ruslan and Kolter, Zico and Heller, Katherine and Weller, Adrian and Oliver, Nuria and Scarlett, Jonathan and Berkenkamp, Felix},
  volume = {235},
  series = {Proceedings of Machine Learning Research},
  month = {21--27 Jul},
  publisher = {PMLR},
  url = {https://proceedings.mlr.press/v235/balsells-rodas24a.html},
  note = {arXiv:2305.15925}
}

@article{stock2002,
  author = {Stock, James H. and Watson, Mark W.},
  title = {Forecasting Using Principal Components From a Large Number of Predictors},
  journal = {Journal of the American Statistical Association},
  year = {2002},
  volume = {97},
  number = {460},
  pages = {1167--1179}
}

@article{andreini2020deep,
  author = {Andreini, Paolo and Izzo, Cosimo and Ricco, Giovanni},
  title = {Deep Dynamic Factor Models},
  journal = {arXiv preprint arXiv:2007.11887},
  year = {2020}
}

@article{sims1980,
  author = {Sims, Christopher A.},
  title = {Macroeconomics and Reality},
  journal = {Econometrica},
  year = {1980},
  volume = {48},
  number = {1},
  pages = {1--48}
}

@article{svar_external_instrument2020,
  author = {Stock, James H. and Watson, Mark W.},
  title = {Identification and Estimation of Dynamic Causal Effects in Macroeconomics Using External Instruments},
  journal = {The Economic Journal},
  year = {2018},
  volume = {128},
  number = {610},
  pages = {917--948}
}

@book{pearl2009causality,
  title = {Causality: Models, Reasoning, and Inference},
  author = {Pearl, Judea},
  year = {2009},
  publisher = {Cambridge University Press},
  edition = {2nd},
  address = {Cambridge, UK}
}

@book{peters2017elements,
  title = {Elements of Causal Inference: Foundations and Learning Algorithms},
  author = {Peters, Jonas and Janzing, Dominik and Scholkopf, Bernhard},
  year = {2017},
  publisher = {MIT Press},
  address = {Cambridge, MA}
}

@book{spirtes2000causation,
  title = {Causation, Prediction, and Search},
  author = {Spirtes, Peter and Glymour, Clark and Scheines, Richard},
  year = {2000},
  edition = {2nd},
  publisher = {MIT Press},
  address = {Cambridge, MA}
}

@article{scholkopf2021causal,
  title = {Toward Causal Representation Learning},
  author = {Scholkopf, Bernhard and Janzing, Dominik and Peters, Jonas and Bauer, Sebastian and Xu, Kun and Bloebaum, Patrick and Zhang, Kun and Mooij, Joris},
  journal = {Proceedings of the National Academy of Sciences},
  volume = {118},
  number = {23},
  pages = {e2021297118},
  year = {2021},
  publisher = {National Academy of Sciences}
}

@article{gu2022efficiently,
  author = {Gu, Albert and Goel, Karan and R{\'e}, Christopher},
  title = {Efficiently Modeling Long Sequences with Structured State Spaces},
  journal = {arXiv preprint arXiv:2111.00396},
  year = {2022}
}

@inproceedings{spacetime2023,
  author = {Zhang, Michael and Saab, Khaled and Poli, Michael and Dao, Tri and Goel, Karan and R{\'e}, Christopher},
  title = {Effectively Modeling Time Series with Simple Discrete State Spaces},
  booktitle = {International Conference on Learning Representations},
  year = {2023}
}

@article{gu2023mamba,
  author = {Gu, Albert and Dao, Tri},
  title = {Mamba: Linear-Time Sequence Modeling with Selective State Spaces},
  journal = {arXiv preprint arXiv:2312.00752},
  year = {2023}
}

@article{liu2023itransformer,
  author = {Liu, Yong and Hu, Tengge and Zhang, Haoran and Wu, Haixu and Wang, Shiyu and Ma, Lintao and Long, Mingsheng},
  title = {iTransformer: Inverted Transformers Are Effective for Time Series Forecasting},
  journal = {arXiv preprint arXiv:2310.06625},
  year = {2023}
}

@article{wang2024timemixer,
  author = {Wang, Shiyu and Wu, Haixu and Shi, Xiaoming and Hu, Tengge and Luo, Huakun and Ma, Lintao and Zhang, James Y. and Zhou, Jun},
  title = {TimeMixer: Decomposable Multiscale Mixing for Time Series Forecasting},
  journal = {arXiv preprint arXiv:2405.14616},
  year = {2024}
}

@article{wang2024timexer,
  author = {Wang, Yuxuan and Wu, Haixu and Dong, Jiaxiang and Qin, Guo and Zhang, Haoran and Liu, Yong and Qiu, Yunzhong and Wang, Jianmin and Long, Mingsheng},
  title = {TimeXer: Empowering Transformers for Time Series Forecasting with Exogenous Variables},
  journal = {Advances in Neural Information Processing Systems},
  volume = {37},
  year = {2024},
  note = {arXiv:2402.19072}
}

@article{krishnan2015dkf,
  author = {Krishnan, Rahul G. and Shalit, Uri and Sontag, David A.},
  title = {Deep Kalman Filters},
  journal = {arXiv preprint arXiv:1511.05121},
  year = {2015}
}

@article{gneiting2007strictly,
  author = {Gneiting, Tilmann and Raftery, Adrian E.},
  title = {Strictly Proper Scoring Rules, Prediction, and Estimation},
  journal = {Journal of the American Statistical Association},
  volume = {102},
  number = {477},
  pages = {359--378},
  year = {2007}
}

@article{matheson1976scoring,
  author = {Matheson, James E. and Winkler, Robert L.},
  title = {Scoring Rules for Continuous Probability Distributions},
  journal = {Management Science},
  volume = {22},
  number = {10},
  pages = {1087--1096},
  year = {1976}
}

@book{lutkepohl2005,
  author = {L\"utkepohl, Helmut},
  title = {New Introduction to Multiple Time Series Analysis},
  publisher = {Springer},
  year = {2005}
}

@book{hamilton1994,
  author = {Hamilton, James D.},
  title = {Time Series Analysis},
  publisher = {Princeton University Press},
  year = {1994}
}

\newpage
\appendix
\raggedbottom

This supplementary material contains proofs (identifiability and degeneracy), experiment details, additional results, and ablation studies.

\section{Identifiability Proof}
\label{app:identifiability}

This appendix gives the full proof of Theorem \ref{thm:temporal_identifiability} (Section \ref{sec:method}). We use the same generative model, Definition \ref{def:identifiability}, and notation as in the main text.

\subsection{Setup and Notation}
\label{app:A:setup}

The generative model includes regime from the outset. Innovations follow the conditional prior in Eq.\ \ref{eq:innovation_prior}, where auxiliary $\mathbf{u}_t$ is observed and regime embedding $\mathbf{e}_t$ is deterministic given $\mathbf{u}_t$. Dynamics follow $\boldsymbol{f}_{t+1} = \bar{A}_t \boldsymbol{f}_t + \bar{B}_t \boldsymbol{\eta}_t$ with diagonal $\bar{A}_t$, $\bar{B}_t$, or the equivalent block-diagonal companion form. Observations satisfy $\boldsymbol{y}_t = g(\boldsymbol{f}_t) + \boldsymbol{\varepsilon}_t$, with $\boldsymbol{\varepsilon}_t \sim \mathcal{N}(\boldsymbol{0}, \sigma^2 I_N)$, injective $g$, and fixed $\sigma^2$.

We say two parameterizations $(\theta, \boldsymbol{\eta}, \boldsymbol{f})$ and $(\tilde{\theta}, \tilde{\boldsymbol{\eta}}, \tilde{\boldsymbol{f}})$ \emph{induce the same observation distribution} if $p_\theta(\boldsymbol{y}_{1:T}, \boldsymbol{u}_{1:T}) = p_{\tilde{\theta}}(\boldsymbol{y}_{1:T}, \boldsymbol{u}_{1:T})$ as densities (or equivalently, the induced laws agree). By the chain rule and the fact that $\boldsymbol{f}_{1:T}$ is a deterministic function of $\boldsymbol{\eta}_{1:T}$ under the dynamics, the joint factorizes as
\begin{equation}
p(\boldsymbol{y}_{1:T}, \boldsymbol{u}_{1:T}, \boldsymbol{\eta}_{1:T}, \boldsymbol{f}_{1:T})
=
p(\boldsymbol{u}_{1:T})\,
\prod_{t=1}^T p(\boldsymbol{\eta}_t \mid \mathbf{u}_t, \mathbf{e}_t)\,
\delta\!\left(\boldsymbol{f}_t - \mathcal{F}_t(\boldsymbol{\eta}_{1:t})\right)\,
p(\boldsymbol{y}_t \mid \boldsymbol{f}_t),
\end{equation}
where $\mathcal{F}_t$ is the deterministic innovation-to-factor map induced by the dynamics and $\delta$ is a Dirac mass enforcing the dynamic constraint (we overload notation across discrete and continuous cases). We will use equality of the observation law together with the iVAE identifiability conditions (injective decoder, sufficient auxiliary variation, and non-Gaussian exponential-family components) to characterize the ambiguity class of $\boldsymbol{\eta}_t$.

\paragraph{Deterministic conditioning variable.}
The regime path enters only through $\mathbf{e}_t = \psi(\mathbf{u}_t)$ with measurable deterministic map $\psi$ (soft assignments or a single regime are both covered). Hence the effective auxiliary variable is $(\mathbf{u}_t,\psi(\mathbf{u}_t))$, which is observed up to deterministic post-processing. This satisfies the exogeneity requirement in the iVAE argument used below.

\subsection{Companion form \texorpdfstring{(AR($p$))}{(AR(p))}}
\label{app:A:companion}
\label{sec:companion_form}

For AR($p$), we stack the current and $p-1$ lagged factor vectors into $\boldsymbol{s}_t = (\boldsymbol{f}_t^\top, \boldsymbol{f}_{t-1}^\top, \ldots, \boldsymbol{f}_{t-p+1}^\top)^\top \in \mathbb{R}^{rp}$. The stacked state follows $\boldsymbol{s}_{t+1} = \mathcal{A}\boldsymbol{s}_t + \mathcal{B}\boldsymbol{\eta}_t$ with block-diagonal $\mathcal{A}$ and $\mathcal{B}$. Unrolling gives
\[
\boldsymbol{s}_t = \mathcal{A}^t \boldsymbol{s}_0 + \sum_{\ell=0}^{t-1} \mathcal{A}^\ell \mathcal{B}\, \boldsymbol{\eta}_{t-1-\ell}.
\]
With $\mathcal{C}$ such that $\boldsymbol{f}_t=\mathcal{C}\boldsymbol{s}_t$, we obtain
\[
\boldsymbol{f}_t = \mathcal{C}\mathcal{A}^t \boldsymbol{s}_0 + \sum_{\ell=0}^{t-1} \mathcal{H}_\ell \boldsymbol{\eta}_{t-1-\ell},
\quad
\mathcal{H}_\ell := \mathcal{C}\mathcal{A}^\ell\mathcal{B}.
\]
Because $\mathcal{A}$ and $\mathcal{B}$ are block diagonal, each $\mathcal{H}_\ell$ is diagonal in factor index. So $\mathcal{F}_t$ is \emph{component-wise}: factor $i$ at time $t$ depends only on $(\eta_{i,1},\ldots,\eta_{i,t})$. This representation also supports efficient long-horizon evaluation, such as FFT or Krylov methods \citep{spacetime2023}.

\paragraph{Component-wise operator form.}
Write $f_{i,t} = \mathcal{F}_{i,t}(\eta_{i,1:t})$ for measurable operators $\mathcal{F}_{i,t}$ induced by diagonal (or block-diagonal companion) dynamics. Then
\[
\boldsymbol{f}_t = \mathcal{F}_t(\boldsymbol{\eta}_{1:t})
=
\bigl(\mathcal{F}_{1,t}(\eta_{1,1:t}),\ldots,\mathcal{F}_{r,t}(\eta_{r,1:t})\bigr)^\top.
\]
This representation is used in Step 2 of the proof (the propagation lemma, Lemma \ref{lem:propagation} in the main text) to show preservation of the same ambiguity class from innovations to factors.

\subsection{Proof of Theorem \ref{thm:temporal_identifiability}}
\label{app:A:proof}

\begin{proof}
Let $(\theta, \boldsymbol{\eta}, \boldsymbol{f})$ and $(\tilde{\theta}, \tilde{\boldsymbol{\eta}}, \tilde{\boldsymbol{f}})$ be two parameterizations that induce the same $p(\boldsymbol{y}_{1:T}, \boldsymbol{u}_{1:T})$. We show that the corresponding latent variables are related by a transformation in $\mathcal{T}$ almost surely.

\medskip
\noindent\textbf{Step 1 (Innovations, invocation of \citealt{khemakhem2020iva}).} The conditioning variable $(\mathbf{u}_t, \mathbf{e}_t)$ is deterministic given $\mathbf{u}_t$ and observed, so it is a valid auxiliary variable in the iVAE setting. The innovation prior has the form
\begin{equation}
p(\boldsymbol{\eta}_t \mid \mathbf{u}_t, \mathbf{e}_t)
=
\prod_{i=1}^r
h_i(\eta_{i,t})
\exp\!\left(
\mathbf{T}_i(\eta_{i,t})^\top \boldsymbol{\lambda}_i(\mathbf{u}_t, \mathbf{e}_t)
-
\Psi_i(\boldsymbol{\lambda}_i(\mathbf{u}_t, \mathbf{e}_t))
\right),
\end{equation}
with $h_i$ the base measure, $\mathbf{T}_i \in \mathbb{R}^k$ the sufficient statistics, and $\Psi_i$ the log-partition function. Under Assumption 1 of Theorem \ref{thm:temporal_identifiability}, the natural-parameter map $\boldsymbol{\lambda}(\mathbf{u}, \mathbf{e})$ has full column rank at at least $rk+1$ distinct conditioning values (reducing to $r+1$ when $k=1$). Therefore, by Theorem 1 of \citet{khemakhem2020iva}, any alternative parameterization with the same observation law must satisfy
\[
\tilde{\eta}_{i,t} = a_i(\eta_{\pi(i),t}) \quad \text{a.s.},
\]
for some permutation $\pi$ on $\{1,\ldots,r\}$ and component-wise affine (or monotone invertible under regularity) maps $\{a_i\}$. We use $a_i$ to denote the ambiguity transformation to avoid re-using $h_i$, which already names the base measure. The rank condition excludes broader linear mixing classes. Hence innovations are partially identifiable up to $\mathcal{T}$.

\noindent\textbf{Step 2 (Propagation to factors, Lemma \ref{lem:propagation}).} By the operator form above, each coordinate satisfies
\[
f_{i,t} = \mathcal{F}_{i,t}(\eta_{i,1:t}),\qquad
\tilde f_{i,t} = \tilde{\mathcal{F}}_{i,t}(\tilde\eta_{i,1:t}).
\]
From Step~1, $\tilde{\eta}_{i,t} = a_i(\eta_{\pi(i),t})$ almost surely. Since dynamics are component-wise and no cross-index coupling is allowed, each transformed coordinate still depends on a single permuted innovation coordinate. Therefore there exist induced maps $\bar{a}_{i,t}$ such that
\[
\tilde f_{i,t}
=
\bar{a}_{i,t}\!\left(f_{\pi(i),t}\right)\quad\text{a.s.}
\]
In the linear diagonal case with affine $a_i$, $\bar{a}_{i,t}$ is affine as well, so the same class $\mathcal{T}$ is preserved from innovations to factors. This fails with full matrices because then $\mathcal{F}_t$ mixes coordinates and the class expands toward general linear transforms. This step is the \emph{only} new identifiability argument in the paper beyond a direct citation of \citet{khemakhem2020iva}.

\noindent\textbf{Step 3 (Observation model).} Equality of $p(\boldsymbol{y}_{1:T},\boldsymbol{u}_{1:T})$ implies equality of each conditional law $p(\boldsymbol{y}_t\mid\mathbf{u}_{1:T})$. By Step~2, latent factors differ only by $\mathcal{T}$. Injectivity of $g$ on factor support excludes additional non-identifiable collapse in observation space. Fixing $\sigma^2$ removes the scale-noise trade-off that would otherwise create an extra ambiguity. Hence the observational model does not enlarge $\mathcal{T}$.

\noindent\textbf{Step 4 (Conclusion).} Combining Steps 1--3, both innovations and factors are identified up to the same class $\mathcal{T}$ whenever two parameterizations induce the same observational law. This is exactly Definition \ref{def:identifiability}.
\end{proof}

\subsection{Remarks}
\label{app:A:remarks}

Identifiability is established at the innovation level through variation of $\boldsymbol{\lambda}$ over $(\mathbf{u}, \mathbf{e})$, then inherited by $\{\boldsymbol{f}_t\}$ through deterministic component-wise dynamics. Companion-form and Krylov implementations are computational devices and do not alter the argument. In particular, $\mathcal{H}_t$ need not be invertible. Under variational inference, identifiability remains a property of the generative model, so the variational distribution may be chosen with any conditioning set (current frame, longer window, or full sequence) without changing the theorem.

\paragraph{What breaks identifiability.}
Two modeling choices are especially problematic. First, conditioning on a continuous latent variable, for example $\boldsymbol{\eta}_t \sim p(\boldsymbol{\eta} \mid \mathbf{u}_t, \mathbf{z}_t)$, breaks the iVAE argument because conditioning is no longer on fixed observed context. Second, directly feeding observations or attention features into innovation definitions, such as $\boldsymbol{\eta}_t = \boldsymbol{\mu}(\mathbf{u}_t, \boldsymbol{y}_{t-L:t}) + \ldots$, makes the auxiliary variable non-exogenous. In both cases, the identification argument no longer applies.

\paragraph{Encoder.} In our implementation the encoder conditions on $(\boldsymbol{y}_t, \mathbf{u}_t)$ and not on $\mathbf{e}_t$ (since $\mathbf{e}_t$ is a deterministic function of $\mathbf{u}_t$). Any richer conditioning is admissible: identifiability is a property of the generative model, so the variational distribution may depend on longer context or on $\mathbf{e}_t$ without changing Theorem \ref{thm:temporal_identifiability}.

\section{Degeneracy of Gaussian Innovations}
\label{app:gaussian_degeneracy}

We present the argument first for the \emph{single-regime} case ($K=1$). In that case, $\mathbf{e}_t$ is fixed, so we write the prior as $p(\boldsymbol{\eta}_t \mid \boldsymbol{u}_t)$. The same degeneracy extends to $K>1$ when conditioning on both $\boldsymbol{u}_t$ and $\mathbf{e}_t$.

\subsection{Setup: Gaussian Innovations in Exponential-Family Form}
\label{app:B:setup}

Suppose innovations are conditionally Gaussian,
\begin{equation}
\boldsymbol{\eta}_t \mid \boldsymbol{u}_t
\sim
\mathcal{N}\!\left(\boldsymbol{\mu}(\boldsymbol{u}_t), \Sigma(\boldsymbol{u}_t)\right),
\end{equation}
with $\Sigma(\boldsymbol{u}_t)$ positive definite. The log-density is
\begin{equation}
\log p(\boldsymbol{\eta}_t \mid \boldsymbol{u}_t)
=
-\tfrac{1}{2}
\boldsymbol{\eta}_t^\top \Sigma(\boldsymbol{u}_t)^{-1} \boldsymbol{\eta}_t
+
\boldsymbol{\eta}_t^\top \Sigma(\boldsymbol{u}_t)^{-1} \boldsymbol{\mu}(\boldsymbol{u}_t)
+
c(\boldsymbol{u}_t),
\end{equation}
where $c(\boldsymbol{u}_t)$ does not depend on $\boldsymbol{\eta}_t$. In exponential-family form, sufficient statistics are $\mathbf{T}(\boldsymbol{\eta}_t) = (\boldsymbol{\eta}_t,\operatorname{vec}(\boldsymbol{\eta}_t \boldsymbol{\eta}_t^\top))$ and natural parameters $\boldsymbol{\lambda}(\boldsymbol{u}_t)$ are functions of $\Sigma(\boldsymbol{u}_t)^{-1}$ and $\Sigma(\boldsymbol{u}_t)^{-1}\boldsymbol{\mu}(\boldsymbol{u}_t)$.

\begin{lemma}[Gaussian family is closed under linear maps]
\label{lem:gaussian_closure}
For any invertible $R \in \mathbb{R}^{r \times r}$, the transformed variable $\tilde{\boldsymbol{\eta}}_t = R \boldsymbol{\eta}_t$ is also Gaussian given $\boldsymbol{u}_t$, with mean $R \boldsymbol{\mu}(\boldsymbol{u}_t)$ and covariance $R \Sigma(\boldsymbol{u}_t) R^\top$. The log-density of $\tilde{\boldsymbol{\eta}}_t \mid \boldsymbol{u}_t$ lies in the same (linear plus quadratic in $\boldsymbol{\eta}$) family as that of $\boldsymbol{\eta}_t \mid \boldsymbol{u}_t$. Thus the observation likelihood $p(\boldsymbol{y}_{1:T} \mid \boldsymbol{u}_{1:T})$ is unchanged under the reparameterization $\boldsymbol{\eta}_t \mapsto R \boldsymbol{\eta}_t$ (with the decoder and dynamics adjusted accordingly).
\end{lemma}
The lemma implies that the ambiguity class of any Gaussian-innovation model includes all invertible linear maps. As a result, full-rank variation in $\boldsymbol{\lambda}(\mathbf{u}, \mathbf{e})$ cannot reduce the class to permutation and component-wise affine maps.

\paragraph{Density transformation detail.}
Let $\tilde{\boldsymbol{\eta}}_t = R\boldsymbol{\eta}_t$ with invertible $R$. By change of variables,
\[
p_{\tilde\eta}(\tilde{\boldsymbol{\eta}}_t \mid \boldsymbol{u}_t)
=
p_{\eta}(R^{-1}\tilde{\boldsymbol{\eta}}_t \mid \boldsymbol{u}_t)\, |\det(R^{-1})|.
\]
Substituting the Gaussian density gives
\[
\tilde{\boldsymbol{\eta}}_t \mid \boldsymbol{u}_t
\sim
\mathcal{N}\!\left(R\boldsymbol{\mu}(\boldsymbol{u}_t),\ R\Sigma(\boldsymbol{u}_t)R^\top\right).
\]
Hence the transformed family remains Gaussian for every invertible $R$.

\subsection{Main Result: Gaussian Innovations Are Not Identifiable}
\label{app:B:mainresult}

\begin{proposition}[Gaussian innovations are not identifiable up to $\mathcal{T}$]
\label{prop:gaussian_degen}
Let $\boldsymbol{\eta}_t \mid \boldsymbol{u}_t \sim
\mathcal{N}(\boldsymbol{\mu}(\boldsymbol{u}_t), \Sigma(\boldsymbol{u}_t))$
with $\Sigma(\boldsymbol{u}_t)$ positive definite. For any invertible
$R \in \mathbb{R}^{r \times r}$, the transformed innovations
$\tilde{\boldsymbol{\eta}}_t = R \boldsymbol{\eta}_t$ satisfy
\begin{equation}
\tilde{\boldsymbol{\eta}}_t \mid \boldsymbol{u}_t
\sim
\mathcal{N}\!\left(R \boldsymbol{\mu}(\boldsymbol{u}_t),\;
R \Sigma(\boldsymbol{u}_t) R^\top \right),
\end{equation}
and, by Lemma \ref{lem:gaussian_closure}, the observation likelihood is unchanged under the corresponding full-model reparameterization. Therefore the ambiguity class includes all invertible linear maps. Innovations are identifiable at most up to invertible linear transformation, not up to the smaller class $\mathcal{T}$ (permutation and component-wise affine maps) that is available in the non-Gaussian case (Appendix \ref{app:identifiability}).
\end{proposition}

Even when $\boldsymbol{\mu}(\boldsymbol{u}_t)$ or $\Sigma(\boldsymbol{u}_t)$ vary with $\boldsymbol{u}_t$, the conditional family remains Gaussian and closed under $R$. Auxiliary variation alone therefore cannot break rotational or scaling symmetry.

\paragraph{Sketch of observational invariance.}
Consider linear dynamics and decoder for clarity,
\[
\boldsymbol{f}_{t+1}=A\boldsymbol{f}_t+B\boldsymbol{\eta}_t,\qquad
\boldsymbol{y}_t=C\boldsymbol{f}_t+\boldsymbol{\varepsilon}_t.
\]
Define transformed parameters
\[
\tilde{\boldsymbol{f}}_t = R\boldsymbol{f}_t,\quad
\tilde A = RAR^{-1},\quad
\tilde B = RB,\quad
\tilde C = CR^{-1}.
\]
Then
\[
\tilde{\boldsymbol{f}}_{t+1}=\tilde A\tilde{\boldsymbol{f}}_t+\tilde B\tilde{\boldsymbol{\eta}}_t,
\qquad
\tilde{\boldsymbol{y}}_t=\tilde C\tilde{\boldsymbol{f}}_t+\boldsymbol{\varepsilon}_t
=
C\boldsymbol{f}_t+\boldsymbol{\varepsilon}_t
=
\boldsymbol{y}_t.
\]
So the induced observation law is unchanged. The same idea extends to nonlinear decoders by composing decoder maps with $R^{-1}$.

\paragraph{Example ($r=2$).}
For $\boldsymbol{\eta}_t \mid \boldsymbol{u}_t \sim \mathcal{N}(\boldsymbol{0}, \Sigma(\boldsymbol{u}_t))$, any orthogonal $Q$ yields $\tilde{\boldsymbol{\eta}}_t = Q \boldsymbol{\eta}_t \sim \mathcal{N}(\boldsymbol{0}, Q \Sigma(\boldsymbol{u}_t) Q^\top)$. The observation distribution is unchanged, so the model cannot distinguish original from rotated innovations.

\subsection{Implications for Dynamic Factor Models}
\label{app:B:implications}

In classical Gaussian dynamic factor models \citep{stock2002}, factors evolve as
\begin{equation}
\boldsymbol{f}_{t+1} = A \boldsymbol{f}_t + \boldsymbol{\eta}_t,
\qquad
\boldsymbol{\eta}_t \sim \mathcal{N}(\boldsymbol{0}, \Sigma),
\end{equation}
with a linear decoder. The joint distribution of observations is then invariant under
$\boldsymbol{f}_t \mapsto R \boldsymbol{f}_t$, $A \mapsto R A R^{-1}$, $\Sigma \mapsto R \Sigma R^\top$
for any invertible $R$---the well-known rotational indeterminacy. Temporal dependence or higher-order (e.g.\ companion-form) dynamics do not remove this ambiguity, because linear Gaussian dynamics remain closed under linear maps.

Restricting $A$ to be diagonal avoids cross-factor mixing in the dynamics but does not restore identifiability. The ambiguity class is fixed at the \emph{innovation} level: any invertible $\tilde{\boldsymbol{\eta}}_t = R \boldsymbol{\eta}_t$ preserves the Gaussian family and can be paired with a transformed decoder and dynamics to yield the same observations. Diagonal $A$ only constrains how factors evolve over time, not the invariance of the innovation distribution under linear mixing.

\subsection{Contrast with Non-Gaussian Innovations and Takeaway}
\label{app:B:contrast}

For non-Gaussian exponential families with linearly independent sufficient statistics (for example Laplace or Student-$t$), the log-density is not invariant to general linear mixing \citep{hyvarinen2000ica}. In this setting, $\tilde{\boldsymbol{\eta}}_t = R \boldsymbol{\eta}_t$ preserves the product form only when $R$ is a permutation of a scaled identity, which corresponds to permutation and component-wise scaling. Conditioning on $\boldsymbol{u}_t$ then yields natural-parameter variation that can satisfy full-rank conditions in iVAE \citep{khemakhem2020iva}. Innovations are then partially identifiable up to permutation and component-wise affine maps (or monotone invertible maps under regularity). iVDFM uses this non-Gaussian route so identifiability is obtained at the innovation level and then propagated through dynamics (Theorem \ref{thm:temporal_identifiability}).

\paragraph{Connection to finite-sample diagnostics.}
This proposition is a population-level identifiability statement. Finite-sample recovery scores such as MCC can still appear similar between Gaussian and non-Gaussian runs, because those scores depend on optimization, sample size, and matching heuristics. Similar finite-sample scores do not imply equal identifiability classes.

The degeneracy of Gaussian innovations is thus a property of the Gaussian family, not of the dynamic structure. iVDFM avoids it by using conditional non-Gaussian innovations.

\section{Experiment Details}
\label{app:experiment_details}

This appendix gives details for the experiments in Section \ref{sec:experiments}: factor recovery, causal intervention, and forecasting.

\subsection{Factor Recovery}
\label{app:C:factor_recovery}

\paragraph{Synthetic DGPs.}
We use two synthetic DGPs adapted from \citet{khemakhem2020iva} and \citet{andreini2020deep}. In the static setting, latents are generated independently across time conditional on an auxiliary variable and observations are formed through nonlinear mixing. In the dynamic setting, latent factors follow AR dynamics driven by i.i.d.\ non-Gaussian innovations (Laplace with unit-variance scaling), and observations are generated by a nonlinear decoder. Both settings use $T=200$, $N=20$, and $r=5$. This design isolates the role of temporal dynamics while keeping measurement structure comparable.

\paragraph{Baselines.}
iVDFM is compared to DFM \citep{stock2002}, DDFM \citep{andreini2020deep}, VAE \citep{kingma2014vae}, and iVAE \citep{khemakhem2020iva}. DFM is a linear dynamic factor model; without further restrictions it is only identifiable up to an arbitrary invertible linear mixing of factors. DDFM uses nonlinear measurement maps but does not impose a controlled ambiguity class. VAE does not use auxiliary conditioning. iVAE applies conditional identification per time step, with time used as auxiliary context and no explicit temporal dependence in latent dynamics.

\paragraph{Metrics.}
Let $F \in \mathbb{R}^{T \times r}$ denote ground-truth factors and $\hat{F} \in \mathbb{R}^{T \times r}$ recovered factors. MCC (higher is better) is computed from max-weight assignment over Pearson correlations between standardized columns of $F$ and $\hat{F}$:
\[
\text{MCC} = \frac{1}{r} \max_{\pi \in S_r} \sum_{i=1}^{r} |C_{i,\pi(i)}|.
\]
Subspace distance (lower is better) is the principal-angle distance between column spans. With orthonormal bases $Q_F$, $Q_{\hat{F}}$ and singular values $\sigma_i$ of $Q_F^\top Q_{\hat{F}}$ clamped to $[0,1]$, we use
\[
\text{Subspace} = \frac{1}{r} \sum_{i=1}^{r} \arccos(\sigma_i).
\]
Smoothness (lower is better) is the average step-to-step $\ell_2$ norm of recovered trajectories. DFM tends to be smoothest because of linear latent dynamics \citep{stock2002}. Trace $R^2$ (higher is better) is the fraction of true factor variance explained by the recovered span \citep{andreini2020deep}, with range $[0,1]$.

\paragraph{Protocol.}
We run $N_{\mathrm{sim}} = 10$ simulations per suite (dynamic and static) with fixed seeds so all models are evaluated on the same data. We report mean and standard deviation for every metric. For stability and comparability, each observed series is standardized to zero mean and unit variance within each synthetic dataset before fitting. Common settings are window $=200$, max epochs $=200$, and learning rate $=0.002$. To avoid one-sided tuning, we apply the same Optuna budget per model within each suite. For iVDFM, VAE, and iVAE, we maximize MCC on a fixed synthetic dataset from the same suite and then evaluate with those selected hyperparameters on the suite simulations. Results appear in Tables \ref{tab:synthetic_factor_recovery_dynamic} and \ref{tab:synthetic_factor_recovery_static}.

\paragraph{Reproducibility details.}
Each simulation uses a fixed triplet of seeds for data generation, model initialization, and minibatch ordering. We reuse the same seed triplets across models within each simulation index. This controls for dataset and optimization randomness when comparing recovery metrics. In dynamic runs, we discard a short warm-up prefix from generated latent trajectories before evaluation to reduce dependence on initialization transients.

\subsection{Causal Intervention}
\label{app:C:causal_intervention}

\paragraph{SCM.}
We use a minimal time-series SCM \citep{pearl2009causality} with exogenous innovations as the only source of randomness. State evolution: $\boldsymbol{f}_t = A \boldsymbol{f}_{t-1} + \boldsymbol{\epsilon}_t$, $\boldsymbol{f}_0 = \mathbf{0}$, with $A \in \mathbb{R}^{r \times r}$ stable. Observation map: $\boldsymbol{y}_t = C \boldsymbol{f}_t$ (base/linear) or $\boldsymbol{y}_t = g(\boldsymbol{f}_t)$ (general). Shocks $\boldsymbol{\epsilon}_t$ are i.i.d.\ non-Gaussian (e.g.\ Laplace), giving the causal graph $\boldsymbol{\epsilon}_t \rightarrow \boldsymbol{f}_t \rightarrow \boldsymbol{y}_t$.

\paragraph{Do-intervention and IRF.}
An intervention at $(t_0, i)$ sets $\epsilon_{t_0}^{(i)} = c$ and leaves all other shocks unchanged. The impulse response function (IRF) at horizon $h$ is the difference in expected observations at $t_0+h$ with and without intervention. Under linear $C$, the ground-truth IRF is
\[
\mathrm{IRF}_{\mathrm{true}}(h) = c \cdot C A^h \mathbf{e}_i \in \mathbb{R}^N,
\]
where $\mathbf{e}_i$ denotes the $i$th standard basis vector in $\mathbb{R}^r$ (distinct from the regime embedding $\mathbf{e}_j$ of the main text).
In the base SCM, $A=\rho I$. In the regime SCM, $A$ depends on the regime. In the chain SCM, $A$ is upper triangular so factor $i$ affects factor $i+1$.

\paragraph{Model-implied IRF (iVDFM).}
iVDFM learns innovations $\boldsymbol{\eta}_t$ and dynamics $\boldsymbol{f}_t = \mathcal{F}_t(\boldsymbol{\eta}_{1:t})$ with observations $\boldsymbol{y}_t = g_\theta(\boldsymbol{f}_t)$. Under the main-text assumptions, $\boldsymbol{\eta}_t$ aligns with SCM shocks $\boldsymbol{\epsilon}_t$ up to the same ambiguity class (Appendix \ref{app:identifiability}). We define model-implied IRFs in five steps. First, fit on observational data. Second, construct $\boldsymbol{\eta}^{\mathrm{do}}$ by setting $(\boldsymbol{\eta}^{\mathrm{do}}_{t_0})^{(i)} = c$. Third, propagate through learned dynamics to obtain $\boldsymbol{f}^{\mathrm{do}}$. Fourth, decode both baseline $\widehat{\boldsymbol{y}}_t$ and intervention trajectory $\boldsymbol{y}^{\mathrm{do}}_t$. Fifth, compute
\[
\widehat{\mathrm{IRF}}(h)=\boldsymbol{y}^{\mathrm{do}}_{t_0+h}-\widehat{\boldsymbol{y}}_{t_0+h}.
\]
This evaluation is specific to models with an interpretable innovation process.

\paragraph{Metrics and protocol.}
We compare $\widehat{\mathrm{IRF}}$ and $\mathrm{IRF}_{\mathrm{true}}$ over horizon-variable grids. IRF-MSE (lower is better) is mean squared error on that grid. Sign accuracy (higher is better) is the fraction of $(h,n)$ pairs where $\mathrm{sign}(\widehat{\mathrm{IRF}}_{h,n})=\mathrm{sign}(\mathrm{IRF}_{\mathrm{true},h,n})$. We also report IRF-MAE and Pearson correlation (Appendix Table \ref{tab:causal_intervention_full}). Sign accuracy is a strict criterion and can be more brittle than MSE/MAE under short sequences because intervention-coordinate alignment is estimated from finite data. We generate data from the three SCM variants and train iVDFM only on observational samples with fixed hyperparameters for reproducibility. We run five simulations per setting. Main text results are in Table \ref{tab:causal_intervention}, and stress tests are in Table \ref{tab:causal_intervention_full}. Experiments use a single GPU.

\paragraph{Intervention setup details.}
Unless otherwise stated, interventions are one-shot impulses at a fixed intervention time $t_0$ with amplitude $c$ chosen from a symmetric range around zero. We evaluate all intervention coordinates and average metrics across coordinates and simulation runs. For each run, baseline and intervention rollouts share the same posterior sample path outside the intervened coordinate so that differences isolate intervention effects.

\subsection{Forecasting}
\label{app:C:forecasting}

\paragraph{Datasets and task.}
We use five real-world multivariate time-series datasets: ETTh1 and ETTh2 (electricity transformer temperature, hourly), ETTm1 and ETTm2 (15-minute), and Weather. For each dataset we train on a fixed train/val split and evaluate out-of-sample forecasts at horizons $H \in \{96, 192, 336, 720\}$. Predictions are made in standardized (scaled) space: we fit a scaler on the training set and evaluate metrics without inverse-transforming, so that MSE and CRPS are comparable across variables and datasets.

\paragraph{Baselines.}
iVDFM is compared to iTransformer \citep{liu2023itransformer}, TimeMixer \citep{wang2024timemixer}, TimeXer \citep{wang2024timexer}, DDFM \citep{andreini2020deep}, and MLPMultivariate. All baselines are trained and evaluated on identical splits and horizons. iVDFM uses fixed per-dataset hyperparameters for reproducibility, with no Optuna during evaluation. Configuration details are in Appendix \ref{app:ablation_study}.

\paragraph{Metrics.}
CRPS (lower is better) is computed from quantile forecasts (10th, 50th, 90th percentiles) and averaged across horizon and variables. Standardized MSE (lower is better) is computed between median point forecasts and targets in standardized space. We report both metrics per dataset, model, and horizon. Table \ref{tab:forecasting} gives horizon-averaged values by dataset, and Table \ref{tab:forecasting_full} reports all horizons.

\paragraph{Protocol.}
Common settings are window $=96$, stride $=25$, and up to 10 rolling forecast origins per dataset-horizon pair. Each model is trained once for each dataset and horizon, and metrics are aggregated over origins. iVDFM encoder, decoder, prior architecture, and training hyperparameters are fixed per dataset in the experiment configuration. Results are reported in Table \ref{tab:forecasting} (main text) and Table \ref{tab:forecasting_full} (appendix).

\paragraph{Evaluation consistency checks.}
Before reporting, we verify that each model uses identical scaler statistics from the training split and that rolling origins are aligned across models. We also check that all reported CRPS values are computed from the same quantile set and that point forecasts for MSE correspond to the median prediction used in CRPS summaries.

\section{Additional Results}
\label{app:additional_results}
\vspace{-0.5em}

\subsection{Synthetic Factor Recovery Visualization}
\label{app:synthetic_factor_recovery_full}

This appendix provides full-page visualizations of synthetic factor recovery for both \textbf{dynamic} and \textbf{static} DGPs. Each grid shows recovered factors for each model together with ground-truth factors. Columns are matched using the same MCC assignment used in evaluation. For readability, recovered factors are rescaled to the mean and standard deviation of true factors. This visual rescaling does not change correlations or reported metrics.

\begin{figure}[htbp]
  \centering
  \begin{subfigure}[c]{\linewidth}
    \centering
    \includegraphics[width=\linewidth,height=0.44\textheight,keepaspectratio]{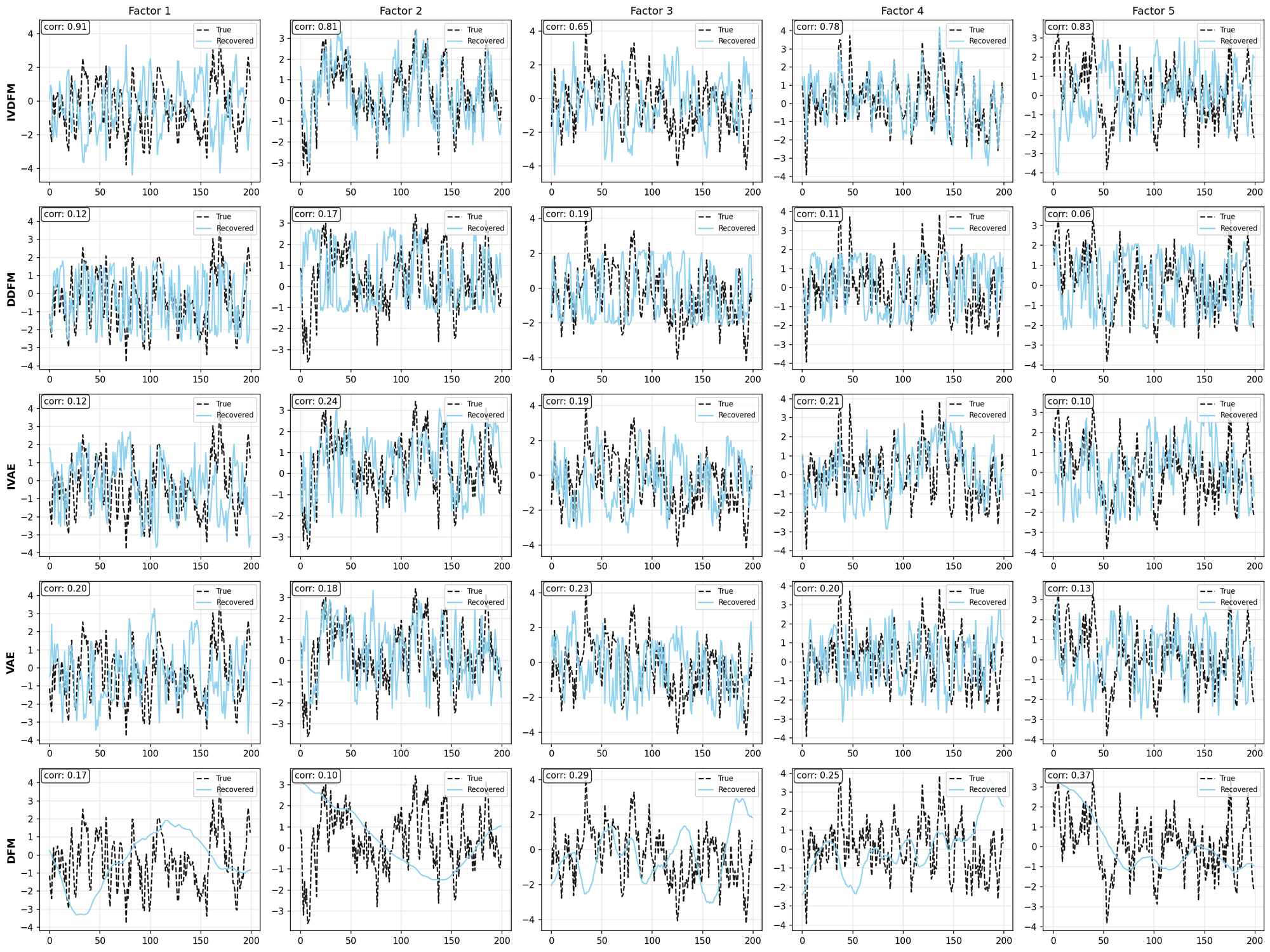}
    \caption{Dynamic synthetic factor recovery (5$\times$5). Columns aligned by MCC-based matching; recovered factors rescaled to true factor scale for visualization.}
    \label{fig:dynamic_factor_recovery_5x5}
  \end{subfigure}
  \vspace{0.5em}
  \begin{subfigure}[c]{\linewidth}
    \centering
    \includegraphics[width=\linewidth,height=0.44\textheight,keepaspectratio]{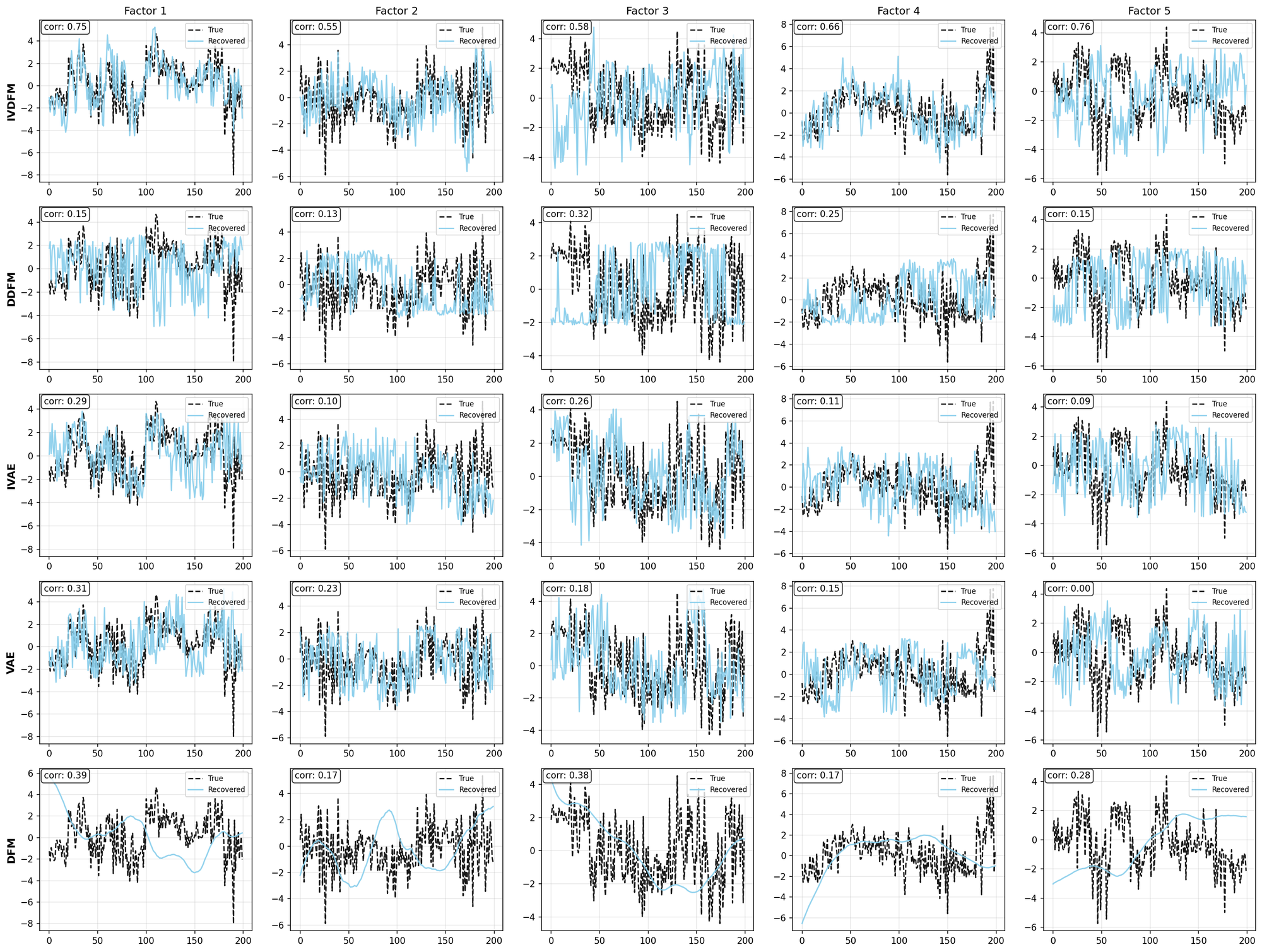}
    \caption{Static synthetic factor recovery (5$\times$5). Columns aligned by MCC-based matching; recovered factors rescaled to true factor scale for visualization.}
    \label{fig:static_factor_recovery_5x5}
  \end{subfigure}
\end{figure}

\subsection{Factor Recovery Diagnostics}
\label{app:D:factor_recovery_diagnostics}

Factor-recovery metrics (MCC, subspace distance, Trace $R^2$) measure alignment to ground truth after permutation-invariant matching. They are affected by finite samples and optimization. They do not directly measure the size of the latent ambiguity class. The diagnostics below therefore add three complementary checks: iVAE-aligned static runs, finite-sample sweeps, and innovation-distribution comparisons.

\paragraph{iVAE static diagnostic.}
This diagnostic tests whether iVAE performance is consistent when the synthetic setup better matches a canonical iVAE-aligned regime and training choices. We therefore run a static diagnostic with longer samples and simpler mixing: $T{=}2000$, $N{=}20$, $r{=}5$, $n_{\mathrm{seg}}{=}20$, and a single linear mixing layer without intermediate nonlinearity. We train VAE and iVAE for 400 epochs and report mean $\pm$ std over 5 simulations.

\begin{table}[t]
    \centering
    \caption{iVAE static diagnostic (static synthetic DGP, iVAE-aligned). Mean $\pm$ std over 5 simulations.}
    \label{tab:ivae_static_sanity}
    \begin{tabular}{lcccc}
    \toprule
    Model & MCC $\uparrow$ & Subspace $\downarrow$ & Smoothness $\downarrow$ & Trace $R^2$ $\uparrow$ \\
    \midrule
    VAE  & $0.7160 \pm 0.0609$ & $0.2665 \pm 0.1176$ & $2.2715 \pm 0.0286$ & $0.8681 \pm 0.0638$ \\
    iVAE & $0.7265 \pm 0.0828$ & $0.2305 \pm 0.0915$ & $2.2710 \pm 0.0635$ & $0.8869 \pm 0.0374$ \\
    \bottomrule
    \end{tabular}
\end{table}

\paragraph{Finite-sample sweep (varying $T$).}
To assess finite-sample effects in the dynamic suite, we vary the sequence length $T$ while keeping $N{=}20$, $r{=}5$, and Laplace innovations. We set the training window equal to $T$ so each simulation contributes one full trajectory, which avoids confounding the finite-sample comparison with a fixed shorter window. Table~\ref{tab:ivdfm_finite_sample_sweep} reports iVDFM performance (mean $\pm$ std over 5 simulations) for $T \in \{100,200,500,1000\}$. Subspace distance and smoothness decrease and Trace~$R^2$ increases with $T$; MCC trends upward with some run-to-run noise at the larger end.\vspace{-0.25em}

\begin{table}[t]
    \centering
    \caption{Dynamic factor recovery finite-sample sweep for iVDFM (Laplace innovations), training window equal to $T$. Mean $\pm$ std over 5 simulations.}
    \label{tab:ivdfm_finite_sample_sweep}
    \begin{tabular}{ccccc}
    \toprule
    $T$ & MCC $\uparrow$ & Subspace $\downarrow$ & Smoothness $\downarrow$ & Trace $R^2$ $\uparrow$ \\
    \midrule
    100  & $0.608 \pm 0.088$ & $0.597 \pm 0.112$ & $2.050 \pm 0.037$ & $0.731 \pm 0.074$ \\
    200  & $0.666 \pm 0.036$ & $0.377 \pm 0.065$ & $1.723 \pm 0.099$ & $0.857 \pm 0.043$ \\
    500  & $0.660 \pm 0.039$ & $0.307 \pm 0.021$ & $1.665 \pm 0.040$ & $0.888 \pm 0.017$ \\
    1000 & $0.703 \pm 0.092$ & $0.337 \pm 0.044$ & $1.661 \pm 0.025$ & $0.857 \pm 0.041$ \\
    \bottomrule
    \end{tabular}
\end{table}

\paragraph{Gaussian vs non-Gaussian innovations (diagnostic).}
\begin{sloppypar}
Our identifiability discussion (Appendix \ref{app:gaussian_degeneracy}) shows that Gaussian innovations admit a larger ambiguity class at the population level. The main factor-recovery metrics, however, measure matched alignment and do not directly quantify ambiguity-class size. As a diagnostic, Table \ref{tab:innovation_distribution_ablation} compares dynamic factor recovery for iVDFM in three settings: (i) Laplace shocks in the synthetic DGP (baseline), (ii) Gaussian shocks in the synthetic DGP, and (iii) Laplace shocks with a Gaussian innovation prior in iVDFM. In this synthetic setup, aggregate recovery metrics are of similar scale.
\end{sloppypar}\vspace{-0.25em}

\begin{table}[t]
    \centering
    \small
    \setlength{\tabcolsep}{3pt}
    \caption{Innovation-distribution diagnostics for dynamic factor recovery (iVDFM). Mean $\pm$ std over 5 simulations.}
    \label{tab:innovation_distribution_ablation}
    \begin{tabular}{lcccc}
    \toprule
    Setting & MCC $\uparrow$ & Subspace $\downarrow$ & Smoothness $\downarrow$ & Trace $R^2$ $\uparrow$ \\
    \midrule
    DGP Laplace; iVDFM Laplace & $0.6696 \pm 0.0696$ & $0.4455 \pm 0.0842$ & $1.6905 \pm 0.0810$ & $0.8103 \pm 0.0471$ \\
    DGP Gaussian; iVDFM Laplace & $0.6879 \pm 0.0769$ & $0.3315 \pm 0.0387$ & $1.7547 \pm 0.0495$ & $0.8742 \pm 0.0279$ \\
    DGP Laplace; iVDFM Gaussian & $0.6805 \pm 0.0779$ & $0.3531 \pm 0.0489$ & $1.7221 \pm 0.0767$ & $0.8720 \pm 0.0319$ \\
    \bottomrule
    \end{tabular}
\end{table}

\subsection{Causal Intervention Stress Tests (Full Table)}
\label{app:causal_intervention_full}

Table \ref{tab:causal_intervention_full} reports causal intervention stress tests (SCM comparison and varying $T$, misspecified $r$ for base, regime, and chain SCMs) for iVDFM. Main-text Table \ref{tab:causal_intervention} reports the SCM comparison only.

\begin{table}[H]
    \centering
    \small
    \caption{iVDFM causal intervention stress tests: varying $T$ and misspecified $r$ (base, regime, chain SCMs). IRF-MSE, IRF-MAE: mean squared / absolute error; Sign Acc: fraction of (time step, variable) pairs with correct sign; IRF Corr: Pearson correlation with ground truth. Mean $\pm$ std over five simulations.}
    \label{tab:causal_intervention_full}
    \begin{tabular}{lcccc}
    \toprule
    Setting & IRF-MSE $\downarrow$ & IRF-MAE $\downarrow$ & Sign Acc $\uparrow$ & IRF Corr $\uparrow$ \\
    \midrule
    \multicolumn{5}{l}{\textit{SCM comparison ($T=200$, $r=3$)}} \\
    Base (linear) & $0.89 \pm 1.01$ & $0.45 \pm 0.23$ & $0.58 \pm 0.25$ & $0.21 \pm 0.75$ \\
    Regime & $1.28 \pm 0.92$ & $0.77 \pm 0.27$ & $0.60 \pm 0.27$ & $0.13 \pm 0.62$ \\
    Chain & $2.18 \pm 2.40$ & $0.97 \pm 0.40$ & $0.58 \pm 0.33$ & $0.20 \pm 0.62$ \\
    \addlinespace
    \multicolumn{5}{l}{\textit{Varying $T$ (base SCM, $r=3$)}} \\
    $T=100$ & $0.75 \pm 0.54$ & $0.44 \pm 0.16$ & $0.02 \pm 0.01$ & $-0.08 \pm 0.52$ \\
    $T=200$ & $0.89 \pm 1.01$ & $0.45 \pm 0.23$ & $0.58 \pm 0.25$ & $0.21 \pm 0.75$ \\
    $T=500$ & $0.76 \pm 0.72$ & $0.43 \pm 0.19$ & $0.60 \pm 0.28$ & $0.27 \pm 0.71$ \\
    $T=1000$ & $0.77 \pm 0.70$ & $0.43 \pm 0.19$ & $0.60 \pm 0.28$ & $0.24 \pm 0.69$ \\
    \addlinespace
    \multicolumn{5}{l}{\textit{Misspecified $r$ (base SCM, $T=200$, true $r=3$)}} \\
    $r=2$ (fit) & $1.16 \pm 1.17$ & $0.52 \pm 0.28$ & $0.41 \pm 0.30$ & $-0.16 \pm 0.77$ \\
    $r=3$ (fit) & $0.89 \pm 1.01$ & $0.45 \pm 0.23$ & $0.58 \pm 0.25$ & $0.21 \pm 0.75$ \\
    $r=4$ (fit) & $1.01 \pm 0.84$ & $0.50 \pm 0.22$ & $0.41 \pm 0.25$ & $-0.27 \pm 0.58$ \\
    $r=5$ (fit) & $0.79 \pm 0.59$ & $0.47 \pm 0.18$ & $0.42 \pm 0.28$ & $-0.04 \pm 0.81$ \\
    \addlinespace
    \multicolumn{5}{l}{\textit{Varying $T$ (Regime SCM, $r=3$)}} \\
    $T=100$ & $1.32 \pm 1.02$ & $0.77 \pm 0.26$ & $0.03 \pm 0.01$ & $0.25 \pm 0.28$ \\
    $T=200$ & $1.28 \pm 0.92$ & $0.77 \pm 0.27$ & $0.60 \pm 0.27$ & $0.13 \pm 0.62$ \\
    $T=500$ & $0.42 \pm 0.33$ & $0.22 \pm 0.09$ & $0.60 \pm 0.28$ & $0.29 \pm 0.69$ \\
    $T=1000$ & $0.44 \pm 0.31$ & $0.23 \pm 0.09$ & $0.58 \pm 0.24$ & $0.17 \pm 0.64$ \\
    \addlinespace
    \multicolumn{5}{l}{\textit{Misspecified $r$ (Regime SCM, $T=200$, true $r=3$)}} \\
    $r=2$ (fit) & $1.51 \pm 1.20$ & $0.82 \pm 0.34$ & $0.40 \pm 0.28$ & $-0.14 \pm 0.67$ \\
    $r=3$ (fit) & $1.28 \pm 0.92$ & $0.77 \pm 0.27$ & $0.60 \pm 0.27$ & $0.13 \pm 0.62$ \\
    $r=4$ (fit) & $1.43 \pm 1.08$ & $0.81 \pm 0.32$ & $0.41 \pm 0.26$ & $-0.10 \pm 0.66$ \\
    $r=5$ (fit) & $1.24 \pm 0.87$ & $0.77 \pm 0.29$ & $0.44 \pm 0.28$ & $0.10 \pm 0.73$ \\
    \addlinespace
    \multicolumn{5}{l}{\textit{Varying $T$ (Chain SCM, $r=3$)}} \\
    $T=100$ & $1.95 \pm 1.27$ & $0.98 \pm 0.27$ & $0.03 \pm 0.01$ & $0.11 \pm 0.28$ \\
    $T=200$ & $2.18 \pm 2.40$ & $0.97 \pm 0.40$ & $0.58 \pm 0.33$ & $0.20 \pm 0.62$ \\
    $T=500$ & $2.06 \pm 2.20$ & $0.95 \pm 0.37$ & $0.62 \pm 0.34$ & $0.17 \pm 0.58$ \\
    $T=1000$ & $2.07 \pm 2.23$ & $0.95 \pm 0.37$ & $0.62 \pm 0.34$ & $0.11 \pm 0.58$ \\
    \addlinespace
    \multicolumn{5}{l}{\textit{Misspecified $r$ (Chain SCM, $T=200$, true $r=3$)}} \\
    $r=2$ (fit) & $2.39 \pm 2.33$ & $1.03 \pm 0.39$ & $0.49 \pm 0.34$ & $-0.04 \pm 0.68$ \\
    $r=3$ (fit) & $2.18 \pm 2.40$ & $0.97 \pm 0.40$ & $0.58 \pm 0.33$ & $0.20 \pm 0.62$ \\
    $r=4$ (fit) & $2.04 \pm 1.36$ & $0.99 \pm 0.29$ & $0.48 \pm 0.29$ & $0.01 \pm 0.65$ \\
    $r=5$ (fit) & $1.43 \pm 0.67$ & $0.90 \pm 0.24$ & $0.62 \pm 0.32$ & $0.34 \pm 0.50$ \\
    \addlinespace
    \bottomrule
    \end{tabular}
\end{table}

\subsection{Forecasting Results}
\label{app:D:forecasting_full}
\label{app:forecasting_full}

Table \ref{tab:forecasting_full} reports CRPS and MSE (standardized) for all datasets and horizons. Dataset names are shown vertically; horizon is in a separate column.

\begin{sidewaystable}[p]
\centering
\caption{CRPS and MSE (standardized). Rows: dataset (horizon); columns: per-model CRPS and MSE (standardized).}
\label{tab:forecasting_full}
\resizebox{\linewidth}{!}{%
\begin{tabular}{llcccccccccccc}
\toprule
Dataset & H & \multicolumn{2}{c}{iVDFM} & \multicolumn{2}{c}{iTransformer} & \multicolumn{2}{c}{TimeMixer} & \multicolumn{2}{c}{TimeXer} & \multicolumn{2}{c}{DDFM} & \multicolumn{2}{c}{MLP} \\
\cmidrule(lr){3-4} \cmidrule(lr){5-6} \cmidrule(lr){7-8} \cmidrule(lr){9-10} \cmidrule(lr){11-12} \cmidrule(lr){13-14}
 & & CRPS & MSE (std.) & CRPS & MSE (std.) & CRPS & MSE (std.) & CRPS & MSE (std.) & CRPS & MSE (std.) & CRPS & MSE (std.) \\
\midrule
    \multirow{4}{*}{\rotatebox[origin=c]{90}{ETTh1}} & 96  & 0.2683 & 1.0907 & 0.6805 & 1.2977 & 0.2492 & 0.6444 & 0.5518 & 1.1886 & 0.3172 & 1.4191 & 0.2275 & 0.6376 \\
     & 192 & 0.2831 & 1.2684 & 0.6868 & 1.2812 & 0.2610 & 0.7614 & 0.5095 & 1.0815 & 0.3329 & 1.4004 & 0.2359 & 0.6832 \\
     & 336 & 0.2835 & 1.3041 & 0.6936 & 1.3671 & 0.2828 & 0.8134 & 0.5683 & 1.2357 & 0.3390 & 1.2963 & 0.2324 & 0.6647 \\
     & 720 & 0.2939 & 1.3684 & 0.7099 & 1.2374 & 0.3115 & 0.9848 & 0.5394 & 1.2506 & 0.3468 & 1.2502 & 0.2871 & 0.8883 \\
    \midrule
    \multirow{4}{*}{\rotatebox[origin=c]{90}{ETTh2}} & 96  & 0.2595 & 1.0307 & 0.5826 & 1.0791 & 0.2239 & 0.4882 & 0.4684 & 1.0791 & 0.2793 & 1.1222 & 0.1869 & 0.3663 \\
     & 192 & 0.2509 & 0.9018 & 0.5156 & 0.9543 & 0.2291 & 0.3539 & 0.4616 & 0.8881 & 0.2765 & 1.0148 & 0.1699 & 0.3365 \\
     & 336 & 0.2493 & 0.8795 & 0.4381 & 1.0926 & 0.2138 & 0.3688 & 0.4406 & 1.0679 & 0.2694 & 0.8808 & 0.1406 & 0.2302 \\
     & 720 & 0.3126 & 1.2534 & 0.4701 & 0.9460 & 0.2549 & 0.3948 & 0.4371 & 0.9258 & 0.2848 & 1.0320 & 0.1782 & 0.3487 \\
    \midrule
    \multirow{4}{*}{\rotatebox[origin=c]{90}{ETTm1}} & 96  & 0.3237 & 1.2088 & 0.7856 & 1.4427 & 0.2309 & 0.4816 & 0.6996 & 1.5203 & 0.3525 & 1.5346 & 0.1844 & 0.4514 \\
     & 192 & 0.3003 & 1.1049 & 0.7127 & 1.2799 & 0.2020 & 0.4598 & 0.6075 & 1.2303 & 0.3357 & 1.1912 & 0.1934 & 0.5271 \\
     & 336 & 0.3062 & 1.1663 & 0.6483 & 1.1407 & 0.2385 & 0.5738 & 0.5320 & 1.0434 & 0.3394 & 1.0315 & 0.2006 & 0.5211 \\
     & 720 & 0.3300 & 1.4042 & 0.6775 & 1.2881 & 0.2455 & 0.6719 & 0.5521 & 1.2070 & 0.3974 & 1.3520 & 0.2251 & 0.6366 \\
    \midrule
    \multirow{4}{*}{\rotatebox[origin=c]{90}{ETTm2}} & 96  & 0.2975 & 0.8705 & 0.8116 & 2.1123 & 0.2476 & 0.5310 & 0.6167 & 2.2748 & 0.5591 & 3.1659 & 0.1958 & 0.3789 \\
     & 192 & 0.3114 & 0.9570 & 0.6778 & 1.4037 & 0.2329 & 0.4155 & 0.5579 & 1.5041 & 0.4921 & 2.5408 & 0.1805 & 0.3908 \\
     & 336 & 0.3564 & 1.1939 & 0.5755 & 0.9626 & 0.2204 & 0.3914 & 0.4785 & 0.9450 & 0.4235 & 1.9969 & 0.1674 & 0.3612 \\
     & 720 & 0.3934 & 1.3712 & 0.5107 & 0.8413 & 0.2281 & 0.3988 & 0.4710 & 0.8558 & 0.3633 & 1.5285 & 0.1798 & 0.4608 \\
    \midrule
    \multirow{4}{*}{\rotatebox[origin=c]{90}{Weather}} & 96  & 0.1969 & 0.6850 & 0.2949 & 0.5596 & 0.1460 & 0.2936 & 0.2704 & 0.5658 & 0.1362 & 0.3011 & 0.1711 & 0.3376 \\
     & 192 & 0.2296 & 0.9232 & 0.2815 & 0.5242 & 0.1527 & 0.3527 & 0.2512 & 0.5182 & 0.1391 & 0.3517 & 0.1733 & 0.3535 \\
     & 336 & 0.2954 & 1.3521 & 0.3608 & 0.9217 & 0.1675 & 0.3532 & 0.3163 & 0.9116 & 0.1791 & 0.5065 & 0.1889 & 0.4050 \\
     & 720 & 0.3621 & 1.7743 & 0.3256 & 0.8048 & 0.1870 & 0.3630 & 0.2852 & 0.8330 & 0.2022 & 0.5977 & 0.1953 & 0.4044 \\
    \bottomrule
\end{tabular}%
}
\end{sidewaystable}

\section{Ablation Study}
\label{app:ablation_study}

We conduct an iVDFM-only ablation on ETTh1 and selected ETT/Weather splits. All runs use the same pipeline as Section \ref{sec:probabilistic_forecasting}. Unless noted, we report CRPS and standardized MSE at horizon 96 and link results to the main forecasting section.

\subsection{Design Choices and Degenerate Cases}
\label{app:E:design_choices}

\paragraph{Prior and context.}
Using constant context $u_t \equiv u_0$ with a linear prior network yields a degenerate baseline: the conditional innovation prior does not vary across time. This is distinct from Gaussian degeneracy in Appendix \ref{app:gaussian_degeneracy}. We compare Laplace, Gaussian, and Student-$t$ innovations; Gaussian remains degenerate under the identifiability conditions. We use Laplace with time-step context.

\paragraph{Number of regimes $K$.}
The choice $K>1$ was introduced as an expressivity fix, not a theoretical add-on (Section \ref{sec:method}). We swept $K \in \{1,3,5,7,9\}$ on ETTh1 at horizon 96 with three seeds each (Table \ref{tab:k_sweep_etth1}). The improvement of small $K$ over $K=1$ is modest but reproducible: $K=3$ lowers CRPS from $0.265$ to $0.256$ and sMSE from $1.061$ to $1.032$ (about $3\%$ on each score). Further increasing $K$ does not help on this setting; $K=7$ and $K=9$ are worse than $K=1$ in both scores. We therefore do not claim that a larger number of regimes is universally beneficial; the mixture prior is a useful expressivity lever, but the useful range is narrow and setting-dependent. The cross-dataset default \texttt{num\_regimes}$=7$ used in the main forecasting table was selected on an earlier joint ETT/Weather sweep and reflects a compromise across datasets rather than a per-dataset optimum. The trade-off is that mixture components are not themselves identifiable (only their deterministic expectation $\mathbf{e}_t$ is), so we do not read individual regime embeddings as interpretable structural states.

\begin{table}[t]
\centering
\small
\caption{iVDFM $K$-ablation on ETTh1 at horizon 96. Mean $\pm$ std over 3 seeds (2026, 1337, 42). Other hyperparameters fixed per \texttt{config/experiment/forecasting.yaml}.}
\label{tab:k_sweep_etth1}
\begin{tabular}{cccc}
\toprule
$K$ & sMSE $\downarrow$ & CRPS $\downarrow$ & MAE $\downarrow$ \\
\midrule
1 & $1.061 \pm 0.023$ & $0.265 \pm 0.005$ & $0.739 \pm 0.015$ \\
3 & $\mathbf{1.032 \pm 0.071}$ & $\mathbf{0.256 \pm 0.013}$ & $0.745 \pm 0.032$ \\
5 & $1.061 \pm 0.124$ & $0.267 \pm 0.028$ & $0.759 \pm 0.070$ \\
7 & $1.113 \pm 0.083$ & $0.273 \pm 0.016$ & $0.789 \pm 0.045$ \\
9 & $1.114 \pm 0.042$ & $0.274 \pm 0.010$ & $0.791 \pm 0.027$ \\
\bottomrule
\end{tabular}
\end{table}

\paragraph{Decoder and post-hoc correction.}
We compare linear, residual (MLP plus linear), and full-MLP decoders for the observation map. Linear decoding underperforms; a residual decoder improves, so we use the full MLP decoder (hidden size 128, one layer). We do not apply post-hoc residual correction.

\subsection{Hyperparameters and Architecture}
\label{app:E:hyperparameters}

\paragraph{Protocol and optimization.}
Aligning the evaluation protocol with benchmark settings (rolling stride equal to the horizon and max origins 10) improves MSE. We use Laplace innovations, tune the learning rate around $4$--$6\times 10^{-4}$, and use dropout around 0.08. We retain layer normalization, a two-layer encoder/decoder with hidden size 128, and set regime temperature $\tau=0.2$ with \texttt{num\_regimes=7}.

\paragraph{Architecture variants.}
Asymmetric bottlenecks and alternative activations (Swish, LeakyReLU) underperform ReLU; residual connections in encoder/decoder do not help. We keep a factor order of 2 for ETTh1/ETTh2 and 3 for ETTm1/ETTm2.

\subsection{Chain-SCM Diagonal-Violation Sweep}
\label{app:E:chain_sweep}

Section \ref{sec:experiments} reports that IRF recovery degrades when the true SCM violates diagonal dynamics. To quantify this, we sweep the chain-SCM off-diagonal strength (parameter \texttt{chain\_strength} in \texttt{synthetic.causal.CausalChainSynthetic}) and re-run the causal-intervention protocol (5 simulations per setting, $T=200$, $r=3$, \texttt{do\_time}=100, \texttt{do\_value}=3.0). Table \ref{tab:chain_sweep} shows that IRF-MSE grows with the amount of cross-factor mixing introduced by the subdiagonal: a zero off-diagonal reduces to the base SCM case and already has non-trivial error (learned coordinates are only aligned up to $\mathcal{T}$), while at \texttt{chain\_strength}=0.5 the error is an order of magnitude larger. Sign accuracy and IRF correlation decrease more slowly but noisily, consistent with the interpretation that diagonal propagation fails to reproduce cross-factor influence paths while still capturing the direction of the dominant own-factor response.

\begin{table}[t]
\centering
\small
\caption{Chain-SCM sweep. Each row is mean $\pm$ std over 5 simulations; IRF-MSE $\downarrow$, IRF-MAE $\downarrow$, Sign Acc $\uparrow$, IRF-Corr $\uparrow$. Row \texttt{chain\_strength=0.2} is the main-text chain setting.}
\label{tab:chain_sweep}
\begin{tabular}{ccccc}
\toprule
\texttt{chain\_strength} & IRF-MSE & IRF-MAE & Sign Acc & IRF-Corr \\
\midrule
0.0 & $0.89 \pm 0.91$ & $0.45 \pm 0.20$ & $0.58 \pm 0.23$ & $0.21 \pm 0.67$ \\
0.1 & $1.31 \pm 1.39$ & $0.66 \pm 0.28$ & $0.70 \pm 0.29$ & $0.27 \pm 0.59$ \\
0.2 & $2.18 \pm 2.15$ & $0.97 \pm 0.36$ & $0.58 \pm 0.30$ & $0.20 \pm 0.56$ \\
0.3 & $4.02 \pm 2.94$ & $1.44 \pm 0.40$ & $0.61 \pm 0.30$ & $0.15 \pm 0.49$ \\
0.5 & $15.26 \pm 7.15$ & $2.99 \pm 0.73$ & $0.56 \pm 0.31$ & $0.15 \pm 0.34$ \\
\bottomrule
\end{tabular}
\end{table}

\subsection{Counterfactual Distribution under Partial Identifiability}
\label{app:E:counterfactual_crps}

To probe whether the identified innovation coordinates support counterfactual forecasting, we compare counterfactual trajectories of the learned model against the true SCM under a fair protocol: both ensembles share the same observational history up to \texttt{do\_time} and vary only in post-intervention shocks. The true ensemble is generated by re-sampling SCM shocks $\varepsilon_{t>t_0}$ with $\varepsilon_{t_0, d}=c$ fixed; the model ensemble resamples $\boldsymbol{\eta}_{t>t_0}$ from the learned prior with the aligned coordinate pinned to $c$ and decodes. We use the base SCM ($A{=}\rho I$, linear $C$) with $T{=}200$, $r{=}3$, \texttt{do\_time}{=}100, $c{=}3.0$, horizon 20, and $100$ samples per ensemble. Alignment per SCM dimension is chosen by picking the model innovation coordinate whose IRF best correlates with the true IRF (horizon 5).

\begin{table}[t]
\centering
\small
\caption{Counterfactual trajectory comparison under the shared-history protocol. Columns: IRF-alignment correlation between model and true IRF for the selected coordinate; mean energy score of the model vs.\ true counterfactual ensemble (lower is better; the last column gives the energy score between the true no-intervention and true intervened ensembles as a reference scale); cosine between model and true expected causal effects $\mathbb{E}[Y^{do}] - \mathbb{E}[Y^{\mathrm{null}}]$ averaged over horizon.}
\label{tab:counterfactual_crps}
\begin{tabular}{ccccc}
\toprule
SCM dim & align corr & ES (model vs true) & ES (null vs do, true) & effect cosine \\
\midrule
0 & $0.66$ & $4.80$ & $3.82$ & $0.32$ \\
1 & $0.52$ & $4.35$ & $3.44$ & $0.21$ \\
2 & $0.01$ & $5.01$ & $3.37$ & $-0.04$ \\
\bottomrule
\end{tabular}
\end{table}

The model's counterfactual distribution remains further from the true counterfactual than the null reference across all dimensions, i.e.\ the learned decoder does not reproduce the full multivariate distribution of the intervention. However, the directional effect $\mathbb{E}[Y^{do}]-\mathbb{E}[Y^{\mathrm{null}}]$ is partially captured whenever the aligned coordinate has non-trivial IRF correlation (cosine $0.2$--$0.3$ for dims $0,1$; essentially zero for dim $2$, where alignment itself fails). This is consistent with the partial-identifiability story: counterfactual accuracy is upper-bounded by the quality of coordinate alignment to the true SCM basis, so an imperfect finite-sample recovery (MCC around $0.6$--$0.7$ in the dynamic suite) propagates directly to imperfect counterfactual magnitudes.

\subsection{Dataset-Specific Settings}
\label{app:E:dataset_settings}

We fix per-dataset settings in configuration for reproducibility (no Optuna). For ETTh1/ETTh2, we use \texttt{factor\_order=2}, \texttt{decoder\_var\(\approx\)0.02--0.06}, and window length 96--1100. For ETTm1, we use \texttt{factor\_order=3}, window length 2000, \texttt{decoder\_var=0.03}, and \texttt{max\_epochs=60}. ETTm2 and Weather use shared defaults (factor order 3, window 96).

\end{document}